\documentclass{article}
\usepackage[utf8]{inputenc}
\usepackage[T1]{fontenc}
\usepackage{graphicx}
\usepackage{amsmath,amsfonts,amssymb,amsthm}
\usepackage{mathtools}
\usepackage{mathrsfs}
\usepackage{booktabs}
\usepackage{bbold}
\usepackage{bm}
\usepackage[font=footnotesize]{caption}
\usepackage{enumitem}
\usepackage{empheq}
\usepackage{framed}
\usepackage{comment}
\usepackage{fullpage}
\usepackage{soul}
\usepackage{dirtytalk}
\usepackage{dsfont}
\usepackage{bbm}
\usepackage{natbib}
\usepackage[dvipsnames]{xcolor}
\usepackage{subcaption}
\usepackage{xspace}
\usepackage[colorlinks,linkcolor=magenta,citecolor=blue, urlcolor=blue, pagebackref=false,backref=false]{hyperref}

\usepackage{cleveref}

\definecolor{myred}{HTML}{880000}
\definecolor{mygreen}{HTML}{008800}
\definecolor{myblue}{HTML}{000088}
\definecolor{linkblue}{HTML}{0000BB}
\definecolor{myred}{HTML}{880000}
\definecolor{mygreen}{HTML}{008800}
\definecolor{myblue}{HTML}{000088}
\definecolor{linkblue}{HTML}{0000BB}

\newcommand{\Var}{\operatorname{Var}}

\renewcommand{\leq}{\leqslant}
\renewcommand{\geq}{\geqslant}
\renewcommand{\le}{\leqslant}
\renewcommand{\ge}{\geqslant}

\usepackage{xspace}

\newtheorem{assumption}{Assumption}
\theoremstyle{remark}

\usepackage{fullpage}

\usepackage{epsf}
\usepackage{fancyhdr}
\usepackage{graphics}
\usepackage{graphicx}
\usepackage{psfrag}
\usepackage{microtype}
\usepackage{epsf}
\usepackage{fancyhdr}
\usepackage{graphics}
\usepackage{psfrag}
\usepackage{microtype}
\usepackage{subcaption}
\usepackage{algorithmic}
\usepackage{subcaption}
\usepackage{algorithm}
\usepackage{url}

\usepackage{enumitem}
\usepackage{booktabs}
\usepackage{caption}

\usepackage{bm,bbm}
\usepackage{mathtools}
\usepackage{cleveref}

\usepackage[mathscr]{eucal}

\usepackage{scalerel}

\newcommand*\D[0]{\mathcal{D}}

\newcommand*\A[0]{\mathcal{A}}

\newtheoremstyle{named}{}{}{\itshape}{}{\bfseries}{.}{.5em}{\thmnote{#3's }#1}
\theoremstyle{named}

\theoremstyle{plain}

\newtheorem{theorem}{Theorem}
\newtheorem{proposition}{Proposition}

\newtheorem{lemma}{Lemma}

\newtheorem{corollary}{Corollary}
\newtheorem{definition}{Definition}

\newlength{\widebarargwidth}
\newlength{\widebarargheight}
\newlength{\widebarargdepth}

\makeatletter
\long\def\@makecaption#1#2{
        \vskip 0.8ex
        \setbox\@tempboxa\hbox{\small {\bf #1:} #2}
        \parindent 1.5em  
        \dimen0=\hsize
        \advance\dimen0 by -3em
        \ifdim \wd\@tempboxa >\dimen0
                \hbox to \hsize{
                        \parindent 0em
                        \hfil
                        \parbox{\dimen0}{\def\baselinestretch{0.96}\small
                                {\bf #1.} #2
                                }
                        \hfil}
        \else \hbox to \hsize{\hfil \box\@tempboxa \hfil}
        \fi
        }
\makeatother

\long\def\comment#1{}
\definecolor{battleshipgrey}{rgb}{0.52, 0.52, 0.51}
\definecolor{darkgray}{rgb}{0.66, 0.66, 0.66}
\definecolor{darkgreen}{rgb}{0.0, 0.2, 0.13}
\definecolor{darkspringgreen}{rgb}{0.09, 0.45, 0.27}
\definecolor{dukeblue}{rgb}{0.0, 0.0, 0.61}
\definecolor{olivedrab7}{rgb}{0.24, 0.2, 0.12}
\definecolor{darkblue}{rgb}{0.0, 0.0, 0.55}
\definecolor{darkscarlet}{rgb}{0.34, 0.01, 0.1}
\definecolor{candyapplered}{rgb}{1.0, 0.03, 0.0}
\definecolor{ao(english)}{rgb}{0.0, 0.5, 0.0}
\definecolor{applegreen}{rgb}{0.55, 0.71, 0.0}

\newcommand{\E}{\mathbb E}

\newcommand{\simiid}{\stackrel{\mathrm{i.i.d.}}{\sim}}

\newcommand{\fakerefassumelip}[1]{\hyperref[assume:smooth-high-order]{{\color{magenta} {\upshape\textbf (}{\upshape{\textbf{Lip}}#1}{\upshape\textbf )}}} }

\newcommand{\cN}{\mathcal{N}}

\DeclareFontFamily{U}{mathx}{}
\DeclareFontShape{U}{mathx}{m}{n}{<-> mathx10}{}
\DeclareSymbolFont{mathx}{U}{mathx}{m}{n}

\DeclareMathAccent{\widecheck}{0}{mathx}{"71}

\newcommand{\Ind}{\mathbb{I}}

\usepackage{mathbbol}

\long\def\comment#1{}

\usepackage[framemethod=TikZ]{mdframed}

\newenvironment{narrowpara}
  {\par\addvspace{\smallskipamount}\narrower\noindent\ignorespaces}
  {\par\addvspace{\smallskipamount}}

\usepackage{authblk}

\begin{document}

\begin{center}
{\bf{\LARGE{Predicting and improving test-time scaling laws via reward tail-guided search}}}

\vspace*{.2in}
{\large{
 \begin{tabular}{ccc}
  Muheng Li$^{ \dagger}$ &  Jian Qian$^{\diamond}$ &
  Wenlong Mou$^{\dagger}$ 
 \end{tabular}

}

\vspace*{.2in}

 \begin{tabular}{c}
 Department of Statistical Sciences, University of Toronto$^{\dagger}$\\
  Department of AI and Data Science, University of Hong Kong$^{\diamond}$
 \end{tabular}

}

\end{center}

\begin{abstract}
  Test-time scaling has emerged as a critical avenue for enhancing the reasoning capabilities of Large Language Models (LLMs). Though the straight-forward ``best-of-$N$'' (BoN) strategy has already demonstrated significant improvements in performance, it lacks principled guidance on the choice of $N$, budget allocation, and multi-stage decision-making, thereby leaving substantial room for optimization.
While many works have explored such optimization, rigorous theoretical guarantees remain limited.
In this work, we propose new methodologies to predict and improve scaling properties via tail-guided search. 
By estimating the tail distribution of rewards, our method predicts the scaling law of LLMs without the need for exhaustive evaluations.
Leveraging this prediction tool, we introduce Scaling-Law Guided (SLG) Search, a new test-time algorithm that dynamically allocates compute to identify and exploit intermediate states with the highest predicted potential.
We theoretically prove that SLG achieves vanishing regret compared to perfect-information oracles, and achieves expected rewards that would otherwise require a polynomially larger compute budget required when using BoN.
Empirically, we validate our framework across different LLMs and reward models, confirming that tail-guided allocation consistently achieves higher reward yields than Best-of-$N$ under identical compute budgets.
Our code is available at \url{https://github.com/PotatoJnny/Scaling-Law-Guided-search}.
\end{abstract}

\section{Introduction}
\label{sec:intro}
The past years have seen remarkable advancements in the capabilities of Large Language Models (LLMs), particularly in complex reasoning tasks~\citep{openai2023gpt4,deepmind2023sparks}. A key driver of these improvements has been the adoption of test-time scaling strategies, which leverage the stochastic nature of LLM outputs to enhance performance without necessitating further model training~\citep{snell2024scaling,lightman2023let,brown2024large}. Among these strategies, the ``best-of-$N$'' (BoN) approach has emerged as a simple yet effective method. By generating multiple outputs and selecting the one with the highest reward according to a predefined metric, BoN has demonstrated significant performance gains across various tasks~\citep{li2023competition,guo2025deepseek}.

Despite its empirical success, the application of BoN strategies still faces several challenges. On the one hand, while it is known that increasing the number of samples $N$ generally leads to better performance, the precise relationship between $N$ and expected reward remains poorly understood. This gap in understanding hinders the ability to make informed decisions about resource allocation, especially in scenarios where computational budgets are constrained. On the other hand, the BoN strategy does not account for the potential benefits of adaptive sampling, where intermediate outputs could inform subsequent sampling decisions.

The answers to these challenges lie in the ability to predict scaling laws governing the performance of BoN strategies, and to develop algorithms that can effectively leverage these predictions. For example, in training-time settings, explicit scaling laws have been established, linking model, data and compute budget to performance metrics~\citep{kaplan2020scaling,hoffmann2022training}. These laws are used to guide resource allocation and model design, and have been theoretically justified under various setups~\cite{lin2025improved,arous2025learning}. By way of contrast, analogous scaling laws for test-time strategies like BoN remain underexplored, both empirically and theoretically.

In this work, we address these gaps by introducing a novel framework for predicting test-time scaling laws of BoN via tail extrapolation of reward distributions. Building on this predictive capability, we further propose Scaling-Law Guided (SLG) Search, an adaptive algorithm that dynamically allocates computational resources based on predicted rewards. Our contribution includes
\begin{itemize}
  \item A new method for predicting the best-of-$N$ scaling behavior with only $m \ll N$ samples, by modeling the tail distribution of rewards and extrapolating to larger $N$.
  \item A scaling-law guided (SLG) search algorithm that utilizes predicted scaling laws to adaptively allocate sampling budgets in two reasoning tasks. We theoretically show that SLG achieves vanishing regret compared to perfect-information oracles, and attains expected rewards that would otherwise require polynomially larger compute budgets under BoN.
  \item Our empirical evaluation across various LLMs and reward models demonstrates that our tail-guided allocation consistently outperforms BoN under identical compute budgets.
\end{itemize}
The rest of the paper is organized as follows. We first introduce notations and discuss related work. \Cref{sec:problem-setup} formalizes the problem setup. In \Cref{sec:scaling-law-prediction}, we present our method for predicting scaling laws via tail extrapolation. \Cref{sec:two-stage} introduces the SLG search algorithm and its theoretical analysis. Finally, \Cref{sec:experiments} details our empirical evaluations, and \Cref{sec:discussion} concludes with a discussion of future directions.

\paragraph{Notation.} 
For any integer $n \ge 1$, let $[n] = \{1, \dots, n\}$. 
We denote the PDF and CDF of the standard normal distribution by $\phi$ and $\Phi$, respectively, and use $\Phi^{-1}$ for its quantile function. 
We write $\phi(\cdot; \mu, \sigma^2)$ for the density of a Gaussian with mean $\mu$ and variance $\sigma^2$.

\subsection{Related work}
\label{sec:related-work}
\paragraph{Test-time scaling laws.}
Recent research has established inference-time compute as a scalable dimension of model performance, distinct from pre-training parameters.
Empirical studies have demonstrated that quality of final outputs improves as test-time budget increases
, allowing smaller models to compete with larger ones given sufficient budget \citep{wu2024inference,chen2024more,arora2025training,liu2025can,schaeffer2025large}.
Regarding the factors governing these laws, \citet{snell2024scaling} empirically found that the optimal scaling strategy depends heavily on the intrinsic difficulty of the prompt.
Subsequent theoretical works have sought to explain these phenomena, typically by modeling the relationship between solution coverage and budget $N$ \citep{brown2024large,levi2024simple,kazdan2025efficient} or analyzing generalization error under imperfect reward models \cite{huang2025best,di2025best,halder2025demystifying}.
Our work contributes to this theoretical lineage by constructing a more accurate statistical model based on distributional tail property to explain and predict test-time scaling laws.

\paragraph{Test-Time Search Algorithms.}
Existing search strategies range from iterative refinement to complex structured exploration. Classical techniques like self-refinement \citep{madaan2023self,jiang2025pag} improve reasoning loops without relying on external signals.
More recent approaches integrate explicit reward models to guide this search, utilizing uncertainty-aware verification \citep{ye2025uncertainty} or generative feedback (``PRMs that think'') \citep{khalifa2025process} to trigger revisions. 
At the most complex end of the spectrum, methods like Tree of Thoughts (ToT) \citep{yao2023tree} and Monte Carlo Tree Search (MCTS) frameworks such as RAP \citep{hao2023reasoning}, TreeBoN \citep{qiu2024treebon} and LATS \citep{zhou2023language} navigate the reasoning space using reward signals.
However, these methods typically rely on dense supervision from trained Process Reward Models (PRMs) or estimate state values using the empirical mean of lookahead rollouts which lacks theoretical guarantees. 
In contrast, we propose a novel search algorithm that leverages predicted scaling laws to guide resource allocation, offering a framework with provable performance guarantees.

\section{Problem setup}
\label{sec:problem-setup}

We formulate the test-time alignment problem as a decision-making task under a fixed computational budget.

\paragraph{Generation process and states.} 
We model the generation process as a trajectory $x=s_0 \to s_1 \to \dots \to s_k \to y$, where $x$ is the initial prompt, $s_i$ represents an \textit{intermediate state} (e.g., a reasoning step, partial plan, or draft segment), and $y$ is the final terminal response.
From any given state $s$, the language model follows a stochastic policy $\pi_{\text{ref}}(\cdot|s)$ to generate the subsequent state or terminate at a response $y$.
Our objective is to identify the generation trajectory that yields the highest reward response $y^*$ subject to a total computational budget.

\paragraph{Reward distribution and value function.} 
We assume access to a reward function $r(y, x)$ that evaluates the quality of a complete response $y$. 
For any intermediate state $s$ (including the prompt $x$), the stochastic policy induces a distribution over potential future outcomes. Let $Y_s$ denote a random final response obtained by rolling out the policy $\pi_{\text{ref}}$ from state $s$ until termination. 
We denote the distribution of its reward as $R_s = r(Y_s, x) \sim F_s$, and refer to $F_s$ as the \textit{reward distribution} associated with state $s$.

To quantify the potential of a state $s$, we define the \textit{value function} $V_N(s)$ as the expected maximum reward achievable by dedicating a sampling budget of $N$ to that state:
\begin{equation}
  \label{eq:value-function}
  V_N(s) := \mathbb{E}\Big[ \max_{1 \leq i \leq N} R_s^{(i)} \Big], \quad \text{where } R_s^{(i)} \overset{\text{i.i.d.}}{\sim} F_s.
\end{equation}
Intuitively, $V_N(s)$ measures the ``scaling potential" of the state $s$: it is the expected best outcome if we were to dedicate a budget of $N$ to exploring it.

\paragraph{Prediction and search objectives.} 
We can formulate the test-time alignment task under a compute budget using the above framework. We consider the following two problems.

\textit{(1) The prediction problem.} Given a state $s$ and a target budget $N$, our first objective is to estimate the scaling behavior for state $s$ using only a small amount of $m$ samples. Concretely, we seek to construct an estimator $\widehat{V}_N(s)$ that approximates
the value function $V_N(s)$ using only $m \ll N$ i.i.d. reward samples $\mathcal{D} = \{ r_1, r_2, \dots, r_m \}$ from the underlying distribution $F_s$.

\textit{(2) The search problem.} Given an initial prompt $x$ and a total budget $N$, our second objective is to design an algorithm that optimizes the expected reward of the maximal reward response generated within the budget. In particular, within one unit of budget, an algorithm $\A$ can sample from any states already generated (including the initial prompt). The decision of which state to sample next can depend on all previously observed states and their sampled rewards. The expected reward of algorithm $\A$ under budget $N$ is defined as
\begin{equation}
  \label{eq:value-function-algorithm}
    V_{N}(\mathcal{A};x) := \mathbb{E}\Big[\max_{y \in \mathcal{Y}_{\text{out}}} r(y,x) \Big],
\end{equation}
where $\mathcal{Y}_{\text{out}}$ is the set of final responses generated by $\mathcal{A}$ given the initial prompt $x$ and budget $N$.
The goal is to find an algorithm $\A$ that maximizes $V_N(\A;x)$. When it is clear from context, we will abbreviate $V_N(\A;x)$ as $V_N(\A)$.

For the rest of the paper, we first address the prediction problem in \Cref{sec:scaling-law-prediction}, and then leverage the resulting estimator to design a 
resource-efficient search algorithm in \Cref{sec:two-stage}.

\section{Predicting scaling laws via tail extrapolation}
\label{sec:scaling-law-prediction}
Based on above setup, we now present our method for predicting the scaling behavior of \(V_N(s)\) for large $N$ using only $m \ll N$ samples. The key insight is that the behavior of the maximum reward is primarily influenced by the tail of the reward distribution. Therefore, accurately modeling the tail allows us to extrapolate \(V_N(s)\) to larger \(N\).

In this section, we first introduce the tail assumption underlying our approach, supported by empirical evidence. We then describe a method to estimate the tail parameters from observed samples, and finally present theoretical guarantees on the accuracy of our scaling law predictions.

\subsection{Tail assumption and empirical validation}\label{subsec:gaussian-tail-assumption}
As the prediction problem of test-time scaling laws requires estimating reward values beyond the observed sample range, it necessitates modeling assumptions on the reward distribution's upper tail behavior.

\paragraph{An empirical observation:} We start by empirically analyzing the reward distributions obtained from LLM-generated responses. By collecting the reward distributions for various prompts and models, we observe consistent patterns in their tail behavior. In \Cref{fig:reward-tail-plot}, we present the histograms of reward samples ($N=5000$) from two representative examples. 
It can be observed that though the main body of the reward distribution is often skewed, the tail aligns closely with a truncated Gaussian distribution. The goodness-of-fit analysis and the Q-Q plots also support this observation, which is provided in Appendix~\ref{appendix:gaussian-tail-check}.

\begin{figure}[h]
    \centering
    \includegraphics[width=1\textwidth]{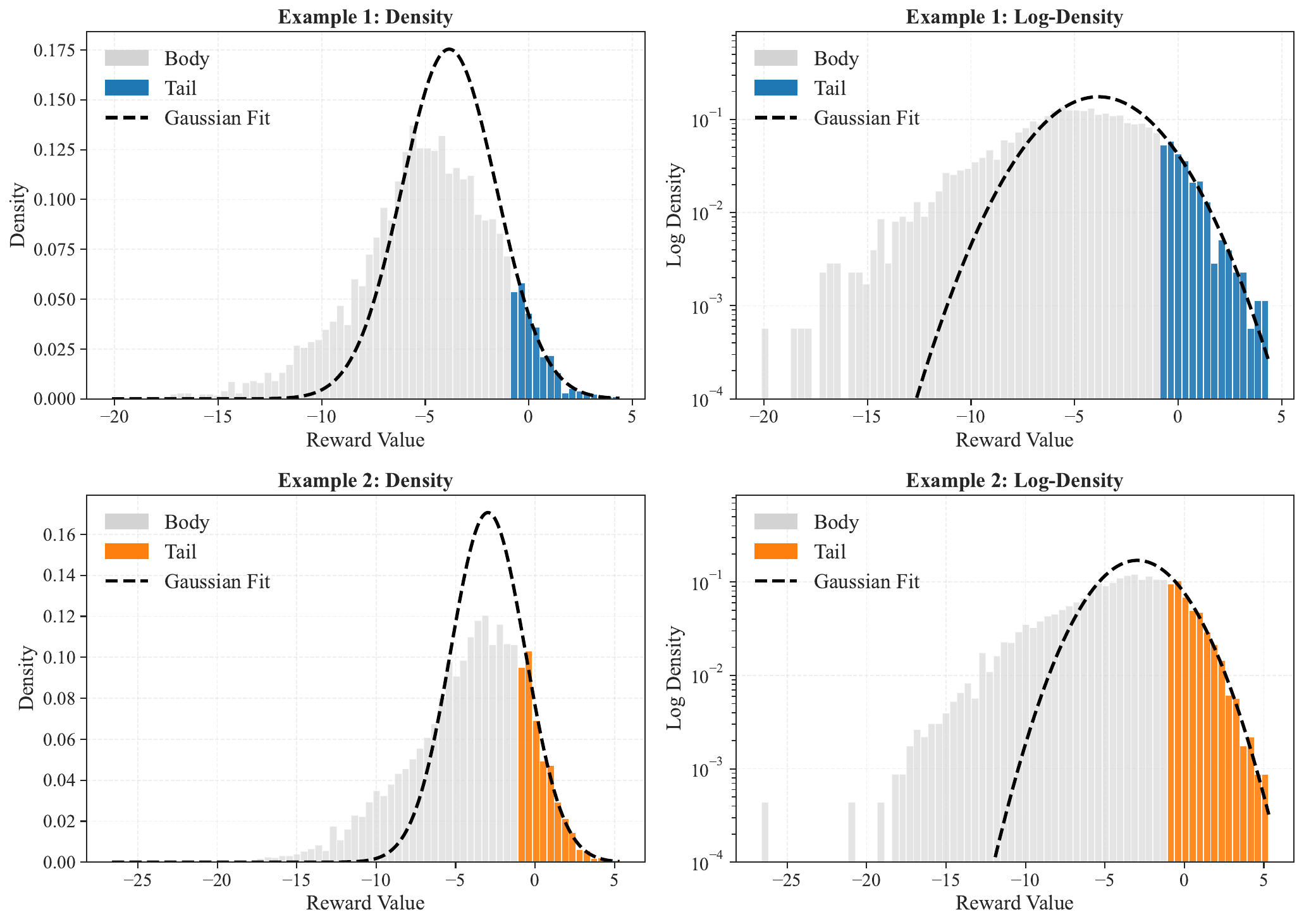} 
    \caption{\textbf{Empirical analysis of reward distributions.} 
    We visualize reward samples ($N=5000$) for two distinct partial responses generated by \texttt{Llama-3.2-1B-Instruct} on an AIME 2024 problem, scored by \texttt{Skywork-Reward-V2}. 
    \textbf{(Left)} The standard density histograms show that the distribution body may contain irregularities or deviations from a perfect bell curve.
    \textbf{(Right)} The log-density plots reveal that despite body noise, the upper tail (highlighted in color) strictly follows a \textbf{parabolic decay}, confirming that the tail behavior is well-approximated by Gaussian.}
    \label{fig:reward-tail-plot}
\end{figure}

 Motivated by these findings, we consider the following Gaussian tail assumption.
\begin{assumption}[Gaussian tail]
\label{assum:gaussian-tail-model}
Let $F_s$ be the cumulative distribution function of the reward distribution at state $s$, and $r_{\alpha} := F_s^{-1}(1-\alpha)$ be the threshold for a tail probability $\alpha \in (0, 1)$. 
We assume that for all $r \geq r_{\alpha}$, the reward density coincides with a Gaussian probability density function:
\begin{equation}
    p_s(r) = \phi(r; \mu, \sigma^2).
\end{equation}
where $\phi(\cdot; \mu, \sigma^2)$ denotes the PDF of a normal distribution with mean $\mu$ and variance $\sigma^2$.
Furthermore, we assume the second moment bound $\mathbb{E}[R_s^2] \leq C_R$.
\end{assumption}

Under Assumption~\ref{assum:gaussian-tail-model}, the full reward density $p_s(r)$ can be viewed as a mixture:
\begin{equation*}
    p_s(r) = \begin{cases} 
      p_{\text{body}}(r) & \text{if } r < r_{\alpha} \\
      \phi(r; \mu, \sigma^2) & \text{if } r \geq r_{\alpha} 
   \end{cases}
\end{equation*}
and we impose no constraints on the distribution's body, allowing the majority of samples to follow an arbitrary $p_{\text{body}}$. 

Under Assumption~\ref{assum:gaussian-tail-model}, the behavior of $V_N(s)$ for large $N$ is governed entirely by the tail parameters $\mu$ and $\sigma$. We now present a method to estimate these parameters and provide theoretical guarantees for the estimation error.

\subsection{Statistical estimation of scaling laws}
We employ the Method of Moments (MoM) on the tail statistics. 
Let $\phi$ and $\Phi$ denote the PDF and CDF of the standard normal distribution, respectively.
Standard results for the truncated normal distribution link the observed tail moments to the generative parameters:
\begin{subequations}
  \label{eq:param-inversion}
\begin{align}
    \mathbb{E}[R \mid R \geq R_{\alpha}] &= \mu + \sigma \lambda(z_{\alpha}), \\
    \text{Var}[R \mid R \geq R_{\alpha}] &= \sigma^2 \left( 1 + z_{\alpha}\lambda(z_{\alpha}) - \lambda(z_{\alpha})^2 \right),
\end{align}
\end{subequations}
where $z_{\alpha} = \Phi^{-1}(1-\alpha)$ and $\lambda(z) = \phi(z) / (1 - \Phi(z))$ is the inverse Mills ratio. 
By inverting these equations, we derive estimates for $\mu$ and $\sigma$ based on the empirical tail mean and variance, and thus
predict $V_N(s)$ via extrapolation. The complete procedure is detailed in Algorithm~\ref{alg:tail-mom}.

\begin{algorithm}[t]
   \caption{Scaling-Law Prediction via Tail Extrapolation}
   \label{alg:tail-mom}
\begin{algorithmic}[1]
   \REQUIRE Reward samples $\mathcal{D}$, tail fraction $\alpha$, budget $N$.  \looseness=-1
   \ENSURE Predicted expected maximum $\hat{V}_N(s)$.
   
   \STATE Compute the empirical threshold $\hat{r}_{\alpha/2}$ as the $(1-\alpha/2)$-quantile of $\mathcal{D}$.
   
   \STATE Extract the tail set $\mathcal{D}_{\text{tail}} = \{ r \in \mathcal{D} \mid r \geq \hat{r}_{\alpha/2} \}$.
   
   \STATE Compute sample mean $\hat{\mu}_{\text{tail}}$ and variance $\hat{\sigma}^2_{\text{tail}}$ of $\mathcal{D}_{\text{tail}}$.
   
   \STATE Estimate parameters $(\hat{\mu}, \hat{\sigma})$ by inverting the truncated normal moments:
   \begin{equation}
       \hat{\sigma} = \sqrt{ \hat{\sigma}^2_{\text{tail}} / {\delta(z)} }, \quad 
       \hat{\mu} = \hat{\mu}_{\text{tail}} - \hat{\sigma} \lambda(z)
   \end{equation}
   where $z = \Phi^{-1}(1-\alpha/2)$ and $\delta(z) = 1 + z\lambda(z) - \lambda(z)^2$, $\lambda(z)$ is the inverse Mills ratio.
   
   \STATE Predict the expected maximum for budget $N$:
   \[
      \hat{V}_N(s) = \hat{\mu} + \hat{\sigma} E(N)
   \]
   where $E(N)$ is the expected maximum of $N$ i.i.d. standard normal variables.

   \STATE \textbf{return} $\hat{V}_N(s)$
\end{algorithmic}
\end{algorithm}

Note that in Algorithm~\ref{alg:tail-mom}, we explicitly restrict estimation to the top $\alpha/2$ quantile rather than the full $\alpha$ tail. 
This conservative threshold mitigates irregularities at the boundary, and enables us to derive the following non-asymptotic error bound for the estimation of $V_N(s)$.

\begin{theorem}
  \label{theorem:error-bound}
  Under Assumptions \ref{assum:gaussian-tail-model}, for any $\delta \in (0,\frac{1}{2})$, when the sample size $m \geq c_1 \log(1/\delta)$,
  then with probability at least $1 - \delta$,
  \begin{align*}
     \left|\hat{V}_N(s) - V_N(s) \right| \leq c_2 \Big\{ \sqrt{\frac{\log N \cdot\log(1/\delta)}{m}} + (1 - \frac{\alpha}{2})^N \Big\},
  \end{align*}
  where
  $c_1, c_2$ are constants only depending on $(\alpha, \sigma, C_R)$.
\end{theorem}
\noindent See \Cref{appendix:proof_error_bound} for the proof of this theorem. \Cref{theorem:error-bound} indicates that the estimation error decays at a rate of $\mathcal{O}(\sqrt{\log N/m})$ as the sample size $m$ increases, 
with an additional exponentially decaying term $(1 - \alpha/2)^N$ that becomes negligible for large $N$.

\Cref{theorem:error-bound} provides theoretical guarantees that allow us to use $\hat{V}_N(s)$ as a reliable signal for decision-making. This predictive capability unlocks several critical applications. We will explore one example application in the following subsection. More importantly, we will leverage this prediction tool to design an adaptive search algorithm in Section~\ref{sec:two-stage}.

\subsubsection{Application to resource allocation}
Let us illustrate one application of our scaling law estimator $\hat{V}_N(s)$ in adaptive resource allocation of computational budget across multiple prompts. \citet{raman2025abon} explores this direction by estimating scaling laws using kernel density estimation (KDE) and demonstrates empirical improvements over uniform allocation. However, it is well known that KDE methods do not extrapolate well beyond the observed sample range (see e.g.~\citet{tsybakov2008introduction}). In contrast, our tail-based method provides robust theoretical guarantees for extrapolation, making it more suitable for high-budget predictions.

Specifically, we consider a total budget of $n_{\text{total}}$ to be allocated across $K$ prompts $\{x_i\}_{i=1}^K$. The goal is to maximize the weighted sum of expected maximum rewards across all prompts, for a given set of weights $\{w_i\}_{i=1}^K$.
\begin{align}
  \max_{\{n_i\}_{i=1}^K} \sum_{i=1}^K w_i V_{n_i}(x_i) \quad 
  \text{s.t.}  \sum_{i=1}^K n_i \leq n_{\text{total}},~ n_i \in \mathbb{N}_+.\label{eq:budget-allocation}
\end{align}
Using our scaling law estimator $\hat{V}_n(x_i)$, we can first allocate a small pilot budget $n_{\text{pilot}} \ll n_{\text{total}}$ to estimate the scaling laws for each prompt, each using $n_{\text{pilot}}/K$ samples. We can then solve the above optimization problem~\eqref{eq:budget-allocation} by replacing $V_{n_i}(x_i)$ with the estimated values $\hat{V}_{n_i}(x_i)$. Note that under our tail assumption, $\hat{V}_{n_i}(x_i)$ is a concave function in $n_i$. By relaxing the integer constraint, the resulting problem becomes a convex optimization that can be solved efficiently using standard methods. Finally, we round the solution to obtain integer allocations.

\section{Scaling-Law Guided search for test-time alignment}
\label{sec:two-stage}

Having established a method to predict the scaling potential of a single state in Section~\ref{sec:scaling-law-prediction}, we now address the search problem defined in Section~\ref{sec:problem-setup}.
Specifically, we focus on the two-stage generation setting, where the model first generates an intermediate state $s$ from the prompt $x$, and then produces a final response $y$ based on $s$. As defined in Section~\ref{sec:problem-setup},
the state $s$ can be interpreted as a chain of thought, a partial response, or one mathematical step.

We recall the search problem from Section~\ref{sec:problem-setup}, where we are given a total budget $N$ to allocate between generating new intermediate states $s$ from the prompt $x$ and sampling final responses $y$ from these states. 
The objective is to maximize the expected maximum reward of the final responses generated within this budget.
To design an efficient search strategy,
we leverage the scaling law prediction derived in Section \ref{sec:scaling-law-prediction}. We begin by extending the local tail assumption to the global state space.

\begin{assumption}
  \label{assume:gaussian-tail}
  We assume Assumption~\ref{assum:gaussian-tail-model} holds for all candidate states $s$ with uniform constants $\alpha$ and $C_R$. Furthermore, we assume that the tail parameters satisfy $\sigma(s) \in [\sigma_{\min}, \sigma_{\max}]$ almost surely for some constants $0 < \sigma_{\min} \leq \sigma_{\max} < \infty$
  , and the latent means $\mu(s)$ have a finite second moment $\mathbb{E}[\mu(s)^2] \leq M_{\mu}$.
\end{assumption}

Since the intermediate states $s$ are generated by a stochastic LLM policy conditioned on the prompt $x$, their associated reward parameters $\mu(s)$ and $\sigma(s)$ are inherently random variables. Assumption~\ref{assume:gaussian-tail} imposes mild regularity conditions on this induced distribution: the bounds on $\sigma(s)$ ensure non-degenerate variance (preventing vanishing or exploding noise), while the moment condition on $\mu(s)$ ensures the distribution of state qualities remains well-behaved.

\subsection{Scaling-Law Guided search algorithm}

Leveraging the scaling law derived in Section \ref{sec:scaling-law-prediction}, we propose an algorithm that dynamically allocates the budget $N$ between exploration (finding good states) and exploitation (resampling the best state), as outlined in Algorithm~\ref{alg:two_stage_search}.

Specifically, the algorithm uses an estimation budget $m$ across $K$ candidate states to gather tail statistics, allowing it to extrapolate their potential to the full budget scale. Based on these predictions, it identifies the single most promising state and concentrates the remaining budget on this trajectory to maximize the final reward.

It is important to note that the SLG algorithm is highly parallelizable -- the exploration phase can be executed concurrently across all $K$ candidate states, and the exploitation phase involves independent sampling from the selected state. Overall, the SLG algorithm only requires two sequential rounds of interaction, making it practically efficient.

\begin{algorithm}[h]
   \caption{Scaling-Law Guided (SLG) search}
   \label{alg:two_stage_search}
\begin{algorithmic}[1]
   \STATE \textbf{Input:} Prompt $x$, Total Budget $N$, Search Width $K$, Estimation Samples $m$. \looseness=-1
   \STATE \textbf{Output:} The highest reward response $y^*$.
   
   \vspace{0.1cm}
   \STATE Generate $K$ intermediate states $\{s_i\}_{i=1}^K$ from prompt $x$.
   \FOR{each state $s_i$}
       \STATE Generate $m$ responses $\{y_{i,j}\}_{j=1}^m$ from $s_i$ and observe their rewards $\{r_{i,j}\}_{j=1}^m$.
       \STATE Estimate $\hat{V}_{N}(s_i)$ using Algorithm~\ref{alg:tail-mom} with $\{r_{i,j}\}_{j=1}^m$.
   \ENDFOR
   \STATE Identify the optimal state: $\hat{I} = \operatorname*{argmax}_{i \in [K]} \hat{V}_N(s_i)$.
   \STATE Allocate remaining budget $N - K \cdot m$ to state $s_{\hat{I}}$ to generate additional responses.
   
   \STATE \textbf{return} the response $y^*$ with the highest reward among all generated responses.
\end{algorithmic}
\end{algorithm}

\subsection{Theoretical results}
\label{subsec:theoretical_setup}

Recall that the performance of a search algorithm $\A$ is measured by the expected reward defined in
Equation~\eqref{eq:value-function-algorithm}, which we restate here for clarity:
\begin{align*}
    V_{N}(\mathcal{A}) := \E\Big[\max_{y \in \mathcal{Y}_{\text{out}}} r(y,x) \Big]
\end{align*}
where the set $\mathcal{Y}_{\text{out}}$ consists of final responses generated by $\A$ using budget $N$.

To contextualize the performance of our SLG algorithm, we introduce two reference strategies for comparison.

\paragraph{Reference strategies.} 
First, we consider a theoretical performance ceiling that operates with perfect knowledge:

\begin{definition}[Full-information oracle $\mathcal{A}^*$]
    The oracle $\mathcal{A}^*$ first samples a pool of $N$ candidate states $\mathcal{S}_N = \{s_1, \dots, s_N\}$ from the prompt $x$. It then identifies the optimal state $s^{*} = \operatorname*{argmax}_{s \in \mathcal{S}_N} V_N(s)$ using perfect distributional information and dedicates a budget of $N$ to sampling from it.
\end{definition}
In other words, the oracle $\A^*$ chooses the best state among $N$ candidates based on perfect information about their reward distributions, and fully exploits it. It is a natural enhancement of BoN with such additional knowledge. Note that $\mathcal{A}^*$ is unattainable in practice as it assumes zero-cost evaluation of the tail parameters. Replicating this strategy blindly would require a budget of at least $2N$ (generating $N$ candidates plus sampling $N$ times from the best), even ignoring the cost of estimation.

Second, we define the standard Best-of-$N$ strategy as a baseline for measuring performance gain:

\begin{definition}[Best-of-$N$ baseline $\mathcal{A}_{\text{BoN}}$]
  \label{def:bon-baseline}
    The Best-of-$N$ strategy samples $N$ responses directly from the prompt $x$ (equivalent to sampling from a random mixture of states) without any intermediate selection.
\end{definition}

We will use the oracle $\A^*$ to bound the \textit{regret} (Section \ref{subsubsec:regret}) and the baseline $\A_{\text{BoN}}$ to quantify the \textit{advantage} in a concrete toy example (Section \ref{subsubsec:gaussian_examples}).

\subsubsection{Regret Analysis}
\label{subsubsec:regret}

We now establish the theoretical guarantees of the SLG algorithm. We first derive a lower bound on the performance of SLG compared to the full-information oracle $\A^*$ under a slightly decreased budget.

\begin{theorem}
\label{thm:regret_bound}
Under Assumption~\ref{assume:gaussian-tail}, consider the SLG algorithm configured with estimation sample size $m$ and search width $K=\lfloor N/2m \rfloor$, consuming a total budget of $N$.
Then, for $N \geq 4(m+1)$ and $m \geq c_0 \log N$ for some constant $c_0 > 0$, we have: 
\begin{equation*}
    V_{N}(\A_{\text{SLG}}) \geq V_{\lfloor N/2m \rfloor}(\mathcal{A}^*) - c \frac{\log N}{\sqrt{m}}
\end{equation*}
where $c$ is a constant depending on model parameters
 $(\alpha, \sigma_{\min}$, $\sigma_{\max}, C_R, M_{\mu})$.
\end{theorem}

\noindent See Appendix~\ref{appendix:proof-regret-bound} for the proof of Theorem~\ref{thm:regret_bound}.

\paragraph{Vanishing Regret.}
By setting $m \asymp (\log N)^\alpha$ with $\alpha > 2$, SLG achieves the same performance as $\A^*$ with a slightly smaller budget $\mathcal{O}(N/(\log N)^\alpha)$, while the estimation error term vanishes as $\mathcal{O}( (\log N)^{1 - \alpha/2} )$.
Under practical settings, this budget reduction is negligible. 
For light-tailed rewards, the oracle $\A^*$ typically satisfies a scaling law of the form $V_N(\A^*) \approx C (\log N)^\gamma$ for some $\gamma \in (0,1)$.\footnote{For instance, the BoN baseline satisfies this with $\gamma=1/2$, as shown in Section~\ref{sec:scaling-law-prediction}.} 
Consequently, the performance loss due to the logarithmic budget reduction scales as:
\begin{align*}
  V_{N}(\A^*) - V_{\frac{N}{ \log^\alpha N}}(\A^*) \asymp \frac{\log \log N}{\log^{1 - \gamma} N},
\end{align*}
which diminishes as $N$ increases. This implies that the gap between SLG and $\A^*$ operating at the \textit{same} budget $N$ also vanishes. Specifically, by choosing $m \asymp (\log N)^3 $, we have:
\begin{equation}
  \label{eq:vanishing-regret}
  V_{N}(\A^*) - V_{N}(\A_{\text{SLG}}) \leq \mathcal{O}\left( \frac{\log \log N}{\log^{1 - \gamma} N} \right).
\end{equation}
We provide the formal statement and proof of this result under mild regularity conditions in Appendix~\ref{appendix:formal_regret_analysis}.

\subsubsection{Comparison to Best-of-N baseline}\label{subsubsec:gaussian_examples}

To explicitly quantify the performance gain of the SLG algorithm, we apply our general framework to a concrete Gaussian example.
This instantiation allows us to derive sharp, non-asymptotic bounds on the advantage gap over the Best-of-N baseline (Definition \ref{def:bon-baseline}).

Suppose the reward distributions for all intermediate states
are Gaussian $p_{s}(r) = \phi(r; \mu(s), \sigma_1)$ with a fixed noise level $\sigma_1 > 0$, and the latent mean parameters are drawn from a Gaussian distribution $\mu(s) \sim \mathcal{N}(\mu_0, \sigma_0^2)$.
We denote the noise ratio as $t = \frac{\sigma_1}{\sigma_0}$.

In this setting, the reward distribution of the prompt simplifies to a Gaussian mixture $\mathcal{N}(\mu_0, \sigma_0^2 + \sigma_1^2)$. And we can explicitly quantify the advantage of our SLG algorithm over the standard Best-of-N baseline:

\begin{proposition}
  \label{prop:homogeneous_gap}
  Consider the setting defined above. Let $t = \frac{\sigma_1}{\sigma_0}$. We set
  the estimation sample size $m = t^2\log N$ and the number of candidates $K = \frac{N}{(1+t)m}$. \footnote{We ignore integer constraints on $m$ and $K$ for clarity of exposition.} For a sufficiently large sampling budget $N$, 
  the expected reward improvement of the SLG algorithm over the Best-of-N baseline is lower-bounded by:
  \begin{equation*}
    V_N(\A_{\text{SLG}}) - V_N(\A_{\text{BoN}})\geq \frac{\sqrt{2}t}{4(t+1)} \sigma_0 \sqrt{\log N}.
  \end{equation*}
\end{proposition}

\noindent See Appendix~\ref{appendix:proof-homogeneous-gap} for the proof of Proposition~\ref{prop:homogeneous_gap}.

Proposition~\ref{prop:homogeneous_gap} shows that SLG strictly outperforms Best-of-$N$ by a margin scaling with $\sqrt{\log N}$. We can translate this additive reward gain into a \textit{sample efficiency} gain, demonstrating that SLG effectively amplifies the compute budget.

\begin{corollary}
\label{cor:compute_amplification}
Under the conditions of Proposition~\ref{prop:homogeneous_gap}. Let $t = \frac{\sigma_1}{\sigma_0}$. The performance of SLG with budget $N$ exceeds that of a Best-of-$N$ baseline operating with a polynomially larger budget $N^{1 + \gamma}$:
\begin{equation*}
    V_N(\mathcal{A}_{\text{SLG}}) \geq V_{N^{1 + \gamma}}(\mathcal{A}_{\text{BoN}}),
\end{equation*}
where we define $\gamma := \frac{\sqrt{2}t}{4(t+1)\sqrt{1+t^2}} > 0$.
\end{corollary}

\noindent See Appendix~\ref{appendix:proof-compute-amplification} for the proof of Corollary~\ref{cor:compute_amplification}.

Corollary~\ref{cor:compute_amplification} implies that SLG search yields a polynomial compute amplification over the Best-of-$N$ baseline. 
The amplification exponent $\gamma$ is a function of the variance ratio $t = \sigma_1/\sigma_0$. Intuitively, if the variance of either the intermediate state generation ($x \to s$) or the final response generation ($s \to y$) vanishes, 
the two-stage structure effectively collapses into a single flat layer, causing SLG to degrade to the Best-of-$N$ baseline.

It is also worth noting that though the scaling potential of a state is mostly determined by its reward variance $\sigma(s)$ (as shown in Section~\ref{sec:scaling-law-prediction}), \Cref{prop:homogeneous_gap} and \Cref{cor:compute_amplification} show that scaling improvements can still be achieved even when all states share the same variance level. This shows the intrinsic value of adaptive state selection based on scaling laws, beyond merely identifying high-variance states.

\section{Experiments}
\label{sec:experiments}

\begin{figure*}[t]
    \centering
    \captionsetup[subfigure]{skip=2pt} 

    \begin{subfigure}[b]{0.325\textwidth}
        \includegraphics[width=\linewidth]{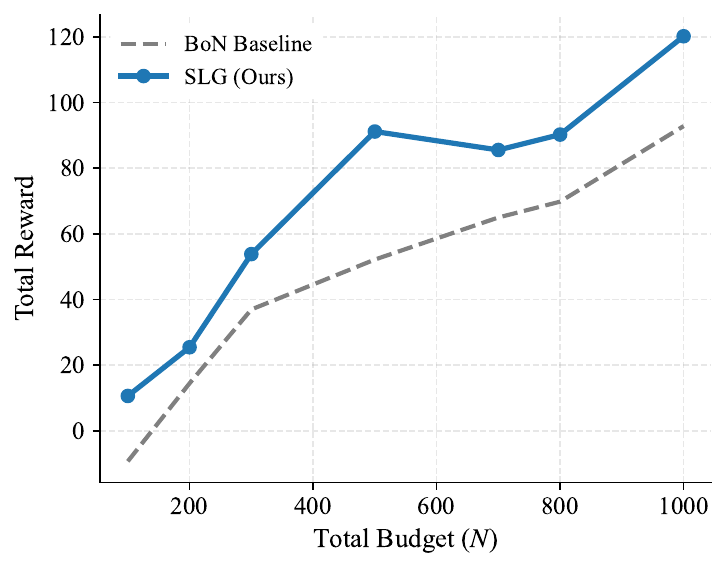}
        \caption{AIME24 (1B)}
        \label{fig:res_1b_aime}
    \end{subfigure}
    \hfill
    \begin{subfigure}[b]{0.325\textwidth}
        \includegraphics[width=\linewidth]{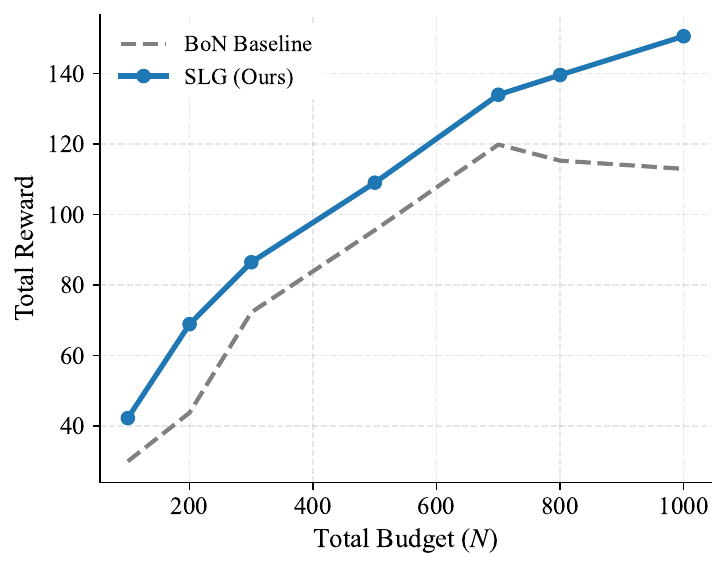}
        \caption{AMC23 (1B)}
        \label{fig:res_1b_amc}
    \end{subfigure}
    \hfill
    \begin{subfigure}[b]{0.325\textwidth}
        \includegraphics[width=\linewidth]{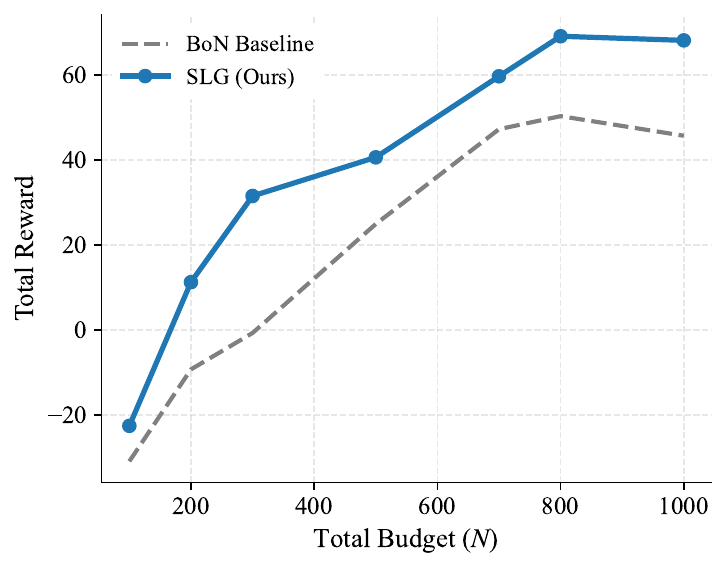}
        \caption{AIME25 (1B)}
        \label{fig:res_1b_aime25}
    \end{subfigure}

    \begin{subfigure}[b]{0.325\textwidth}
        \includegraphics[width=\linewidth]{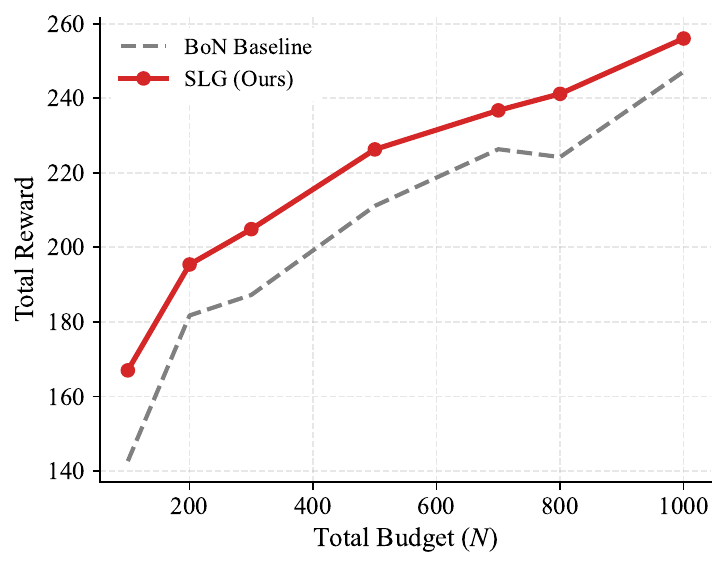}
        \caption{AIME24 (7B)}
        \label{fig:res_7b_aime}
    \end{subfigure}
    \hfill
    \begin{subfigure}[b]{0.325\textwidth}
        \includegraphics[width=\linewidth]{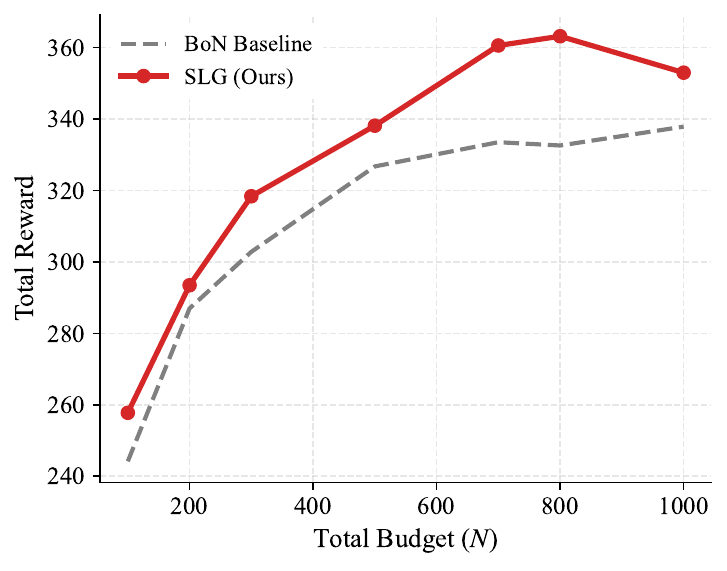}
        \caption{AMC23 (7B)}
        \label{fig:res_7b_amc}
    \end{subfigure}
    \hfill
    \begin{subfigure}[b]{0.325\textwidth}
        \includegraphics[width=\linewidth]{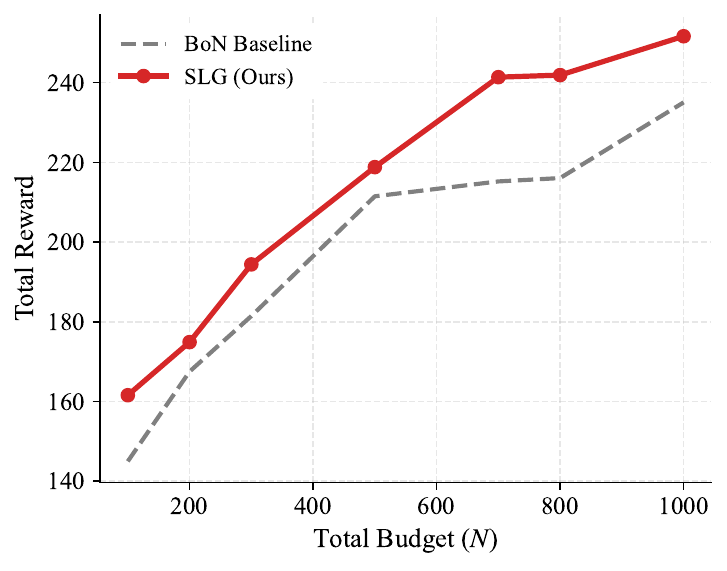}
        \caption{AIME25 (7B)}
        \label{fig:res_7b_aime25}
    \end{subfigure}
    
    \caption{\textbf{Main Results on Test-Time Scaling.} We compare SLG (Solid) vs. BoN (Dashed) across three datasets (\textbf{Columns}) and two model scales (\textbf{Rows}). 
    Here, 1B and 7B denote \texttt{Llama-3.2-1B-Instruct} and \texttt{Qwen2.5-7B-Instruct}.
    SLG demonstrates superior efficiency across all configurations.}
    \label{fig:main_results}
\end{figure*}

In this section, we empirically validate the proposed framework. Our evaluation focuses on two key aspects: 
(1) \textbf{Search efficiency}, assessing if SLG outperforms standard Best-of-$N$ strategies under identical compute budgets; and 
(2) \textbf{Robustness verification}, confirming the stability of SLG across various parameter settings and model choices.

\subsection{Experimental setup}
\label{subsec:setup}

\begin{figure*}[t]
    \centering
    \begin{subfigure}[b]{0.3\textwidth}
        \includegraphics[width=\linewidth]{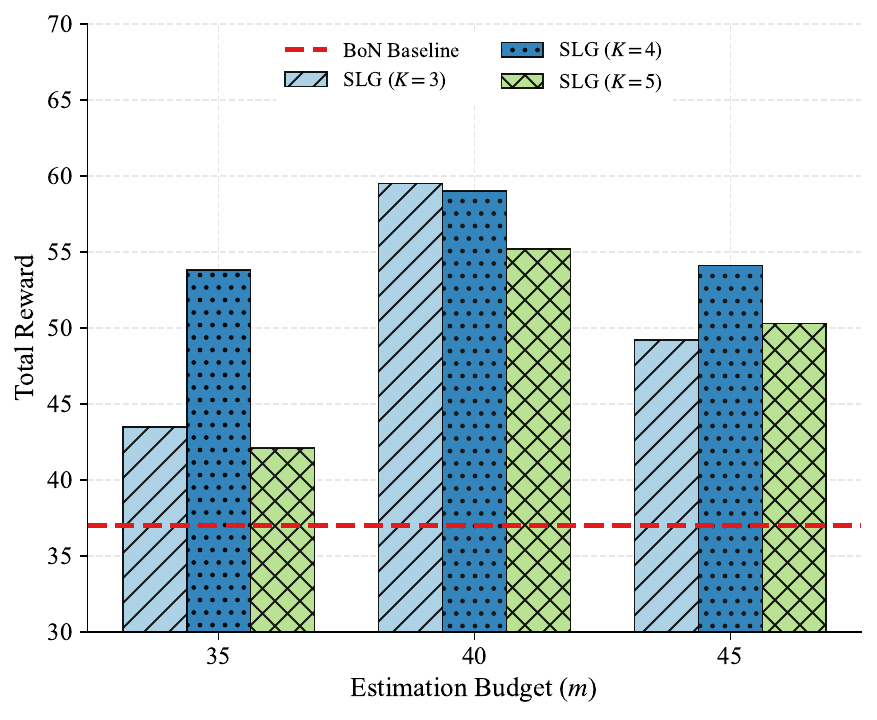}
        \caption{Search Parameters ($K, m$)}
        \label{fig:ablation_km}
    \end{subfigure}
    \hfill
    \begin{subfigure}[b]{0.3\textwidth}
        \includegraphics[width=\linewidth]{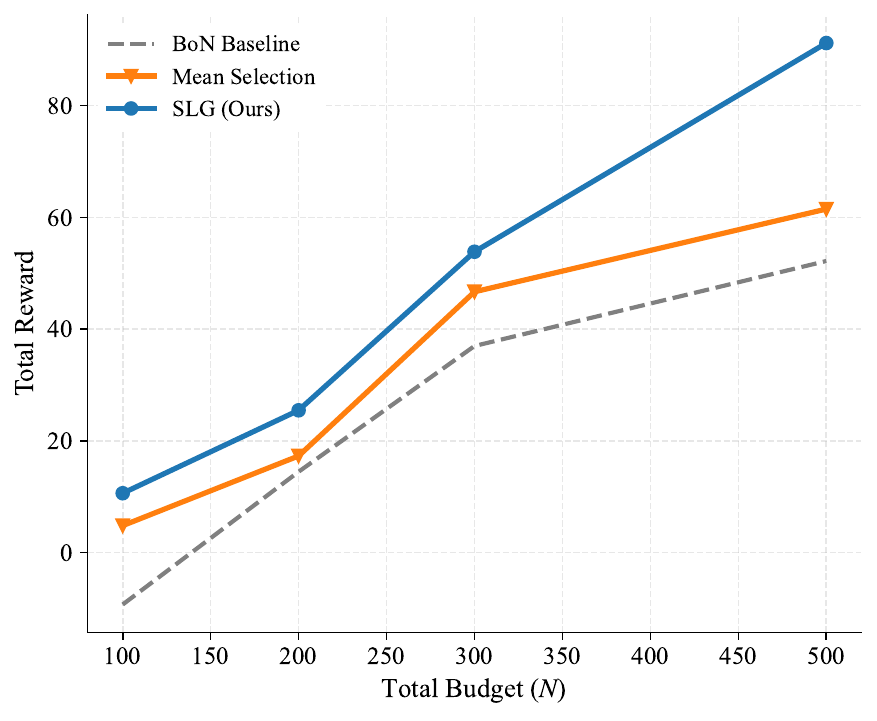}
        \caption{Tail Awareness vs. Mean}
        \label{fig:ablation_tail}
    \end{subfigure}
    \hfill
    \begin{subfigure}[b]{0.3\textwidth}
        \includegraphics[width=\linewidth]{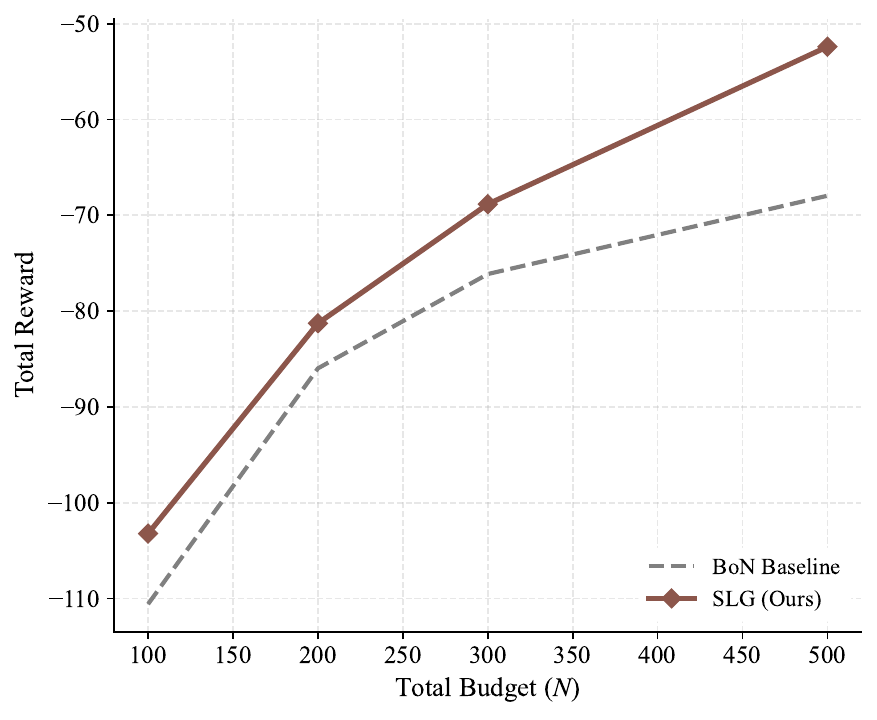}
        \caption{Reward Model Robustness}
        \label{fig:ablation_rm}
    \end{subfigure}
    \caption{\textbf{Ablation Studies.} (a) Impact of search width $K$ and estimation budget $m$, showing a concave trend with a clear optimum. (b) Comparison against a mean-based selection baseline, confirming the necessity of tail-targeting. (c) Verification using the Qwen-based reward model, demonstrating robustness to the feedback signal.}
    \label{fig:ablation_studies}
\end{figure*}

Following recent work on test-time scaling in mathematical reasoning \citep{guha2025openthoughts,agarwal2025first,yu2025limits,wang2025m1}, we evaluate our approach on three verifiable benchmarks targeting advanced problem-solving capabilities:
\textbf{AMC 2023}, a 40-problem subset of the 2023 American Mathematics Competitions assessing non-routine skills in algebra, geometry, and combinatorics \citep{maa2023amc}; 
\textbf{AIME 2024}, consisting of the 30 problems from the 2024 AIME I/II competitions \citep{maa2024aime}; 
and \textbf{AIME 2025}, the analogous set of 30 problems from the most recent 2025 cycle \citep{maa2025aime}.

To provide the feedback signal, we utilize \texttt{Skywork-Reward-V2-Llama-3.1-8B} \citep{liu2025skywork}, a state-of-the-art outcome reward model (ORM) that maps candidate solutions to a scalar quality score $r \in \mathbb{R}$, serving as the oracle for our search optimization.

\paragraph{Baselines and metrics.} 
We primarily compare our Scaling-Law Guided (SLG) search against the standard \textbf{Best-of-$N$ (BoN)} baseline (Definition~\ref{def:bon-baseline}). 
BoN allocates the entire computational budget $N$ to generating independent responses and selecting the candidate with the highest reward. 
To assess performance, we employ the \textbf{total reward} as our primary metric, defined as the accumulated sum of maximum rewards across the test set: 
\[
\mathcal{R}(\mathcal{A}, N) = \sum_{x \in \mathcal{X}} \max_{y \in S_{\mathcal{A}}(x; N)} r(y, x),
\]
where $S_{\mathcal{A}}(x; N)$ denotes the set of candidate responses generated by algorithm $\mathcal{A}$ for prompt $x$ under budget $N$, and $r(y, x)$ is the reward signal provided by the ORM.
This cumulative metric directly measures the aggregate utility under a fixed compute budget, quantifying how efficiently different test-time strategies exploit the reward signal.

\paragraph{Implementation details.} 

We specify the implementation details of our SLG algorithm below. \footnote{Our code is publicly available at \url{https://github.com/PotatoJnny/Scaling-Law-Guided-search}.}

\begin{itemize}
  \item \textbf{State Definition.} We define a state $s$ as the first-100-token partial response generated from the prompt $x$. 
  \item \textbf{Hyperparameters.} We fix the tail fraction $\alpha = 0.2$. Guided by Equation~\ref{eq:vanishing-regret}, we dynamically scale parameters with the budget $N$, setting the estimation budget $m(N) \approx \frac{1}{5}(\ln N)^3$ and search width $K(N) \approx \frac{N}{2m(N)}$. See Appendix~\ref{appendix:experimental_details} for the exact schedule.
\end{itemize}

For the small budget setting where $K \leq m $, we implement a greedy pilot phase to initialize high-quality candidates: we first sample $m$ trajectories $(x \to s \to y)$ and select the top $K$ states $s$ based on their initial outcome rewards for the subsequent scaling law estimation. 
The configuration above, while not strictly optimal for every specific budget, yields consistent gains, indicating that our framework requires minimal hyperparameter tuning.

\subsection{Main Results}
\label{subsec:main_results}

We evaluate SLG against the Best-of-$N$ (BoN) (Figure~\ref{fig:main_results}).

\paragraph{Consistent Superiority.}
SLG strictly outperforms BoN across all budgets and experimental settings. This advantage is robust to scaling: on AIME 2024, SLG yields substantial margins for both the 1B model ($120.2$ vs $92.9$ at $N=1000$) and the stronger 7B model ($256.0$ vs $247.1$), confirming that scaling-law guidance provides additive gains even for capable reasoners. \looseness=-1

\paragraph{Resource Amplification.}
Consistent with Corollary~\ref{cor:compute_amplification}, SLG acts as a compute multiplier. Notably, on AIME 2024 with Qwen-7B, SLG matches the peak performance of BoN ($N=1000$, score $337.9$) using only \textbf{half the budget} ($N=500$, score $338.1$). Similarly, for Llama-1B, SLG ($N=500$) surpasses BoN at $N=700$ ($133.9$ vs $119.8$). This validates that tail-guided search yields superior sample efficiency compared to naive scaling. \looseness=-1

\subsection{Ablation Studies}
\label{subsec:ablation}

We validate core components using \texttt{Llama-3.2-1B-Instruct} on AIME 2024 dataset.

\paragraph{Impact of Search Parameters ($K$ vs. $m$).}
Varying the search width $K$ and pilot budget $m$ at a fixed $N=300$ (Figure~\ref{fig:ablation_km}) demonstrates SLG's robustness; all configurations outperform the baseline. We observe a distinct performance peak at $K=3, m=40$ (59.5 vs. 36.9), which effectively balances estimation cost against exploitation resources. This confirms that identifying the ``sweet spot'' between breadth (exploration) and depth (exploitation) is critical for maximizing compute gain. \looseness=-1

\paragraph{Importance of Tail Awareness.}
Figure~\ref{fig:ablation_tail} compares SLG against a \textbf{Mean-Selection} baseline. SLG consistently dominates, particularly at $N=500$ where it achieves a score of \textbf{91.2} versus \textbf{61.5} for the mean-based approach. This empirically confirms that for reasoning tasks requiring a single best response, explicitly modeling the reward tail is a superior proxy to maximizing expected value. \looseness=-1

\paragraph{Robustness to Reward Models.}
Switching the feedback signal to \texttt{Skywork-Reward-V2-Qwen-8B} (Figure~\ref{fig:ablation_rm}) preserves SLG's significant advantage over the baseline. This indicates that our efficiency gains are intrinsic to the search dynamics and not artifacts of a specific reward model architecture. \looseness=-1

\section{Discussion}
\label{sec:discussion}

In this work, we established a principled framework for test-time scaling that bridges the gap between statistical prediction and algorithmic search. 
We first introduced a methodology to predict the scaling laws of Best-of-$N$ performance via tail extrapolation, allowing us to estimate the ``scaling potential'' of intermediate states without exhaustive sampling. 
Building on this predictive capability, we designed the Scaling-Law Guided (SLG) search, an algorithm that utilizes these forecasts to dynamically allocate compute resources. 
Both our theoretical analysis and empirical results confirm that this tail-guided approach acts as a resource amplifier, consistently achieving higher rewards than the Best-of-$N$ baseline under identical compute budgets.

We view this work as a first step toward a rigorous theory of test-time scaling. Promising directions for future research include: (1) generalizing our framework to recursive settings like MCTS to guide pruning; (2) applying these principles to broader agentic workflows such as code generation; and (3) incorporating the error structure of imperfect reward models into scaling-law predictions.

\newpage
\bibliographystyle{plainnat}  
\bibliography{mybib}
\appendix

\section{Gaussian Tail Check of Reward Distributions}
\label{appendix:gaussian-tail-check}

To rigorously validate the Gaussian Tail Assumption (Assumption~\ref{assum:gaussian-tail-model}), we provide a comprehensive Q-Q (Quantile-Quantile) plot analysis comparing the global distribution fit versus the tail-specific fit.

Figure~\ref{fig:app-qq-plots} presents the diagnostic plots for the two representative examples discussed in Section~\ref{sec:scaling-law-prediction}, and Table~\ref{tab:r2-comparison} summarizes the corresponding goodness-of-fit metrics ($R^2$ values).
\begin{itemize}
    \item \textbf{Top Row (Global Fit):} We fit a standard Gaussian distribution to the entire dataset ($N=5000$). While the fit is generally reasonable, deviations are observable in the extreme quantiles (visualized as data points diverging from the black dashed line), indicating that the body distribution is not perfectly Gaussian, which is also reflected in the lower $R^2$ values (e.g., $0.9707$ for Example 2).
    \item \textbf{Bottom Row (Tail Fit):} We fit a Truncated Gaussian distribution specifically to the top 20\% of the data. The data points align almost perfectly with the theoretical line, yielding significantly higher $R^2$ values.
\end{itemize}

This comparison confirms that while the global reward distribution may be complex, the asymptotic tail behavior—which governs the scaling laws—is strictly well-behaved and accurately modeled by a Truncated Gaussian.

\begin{figure}[h]
    \centering
    \includegraphics[width=0.95\textwidth]{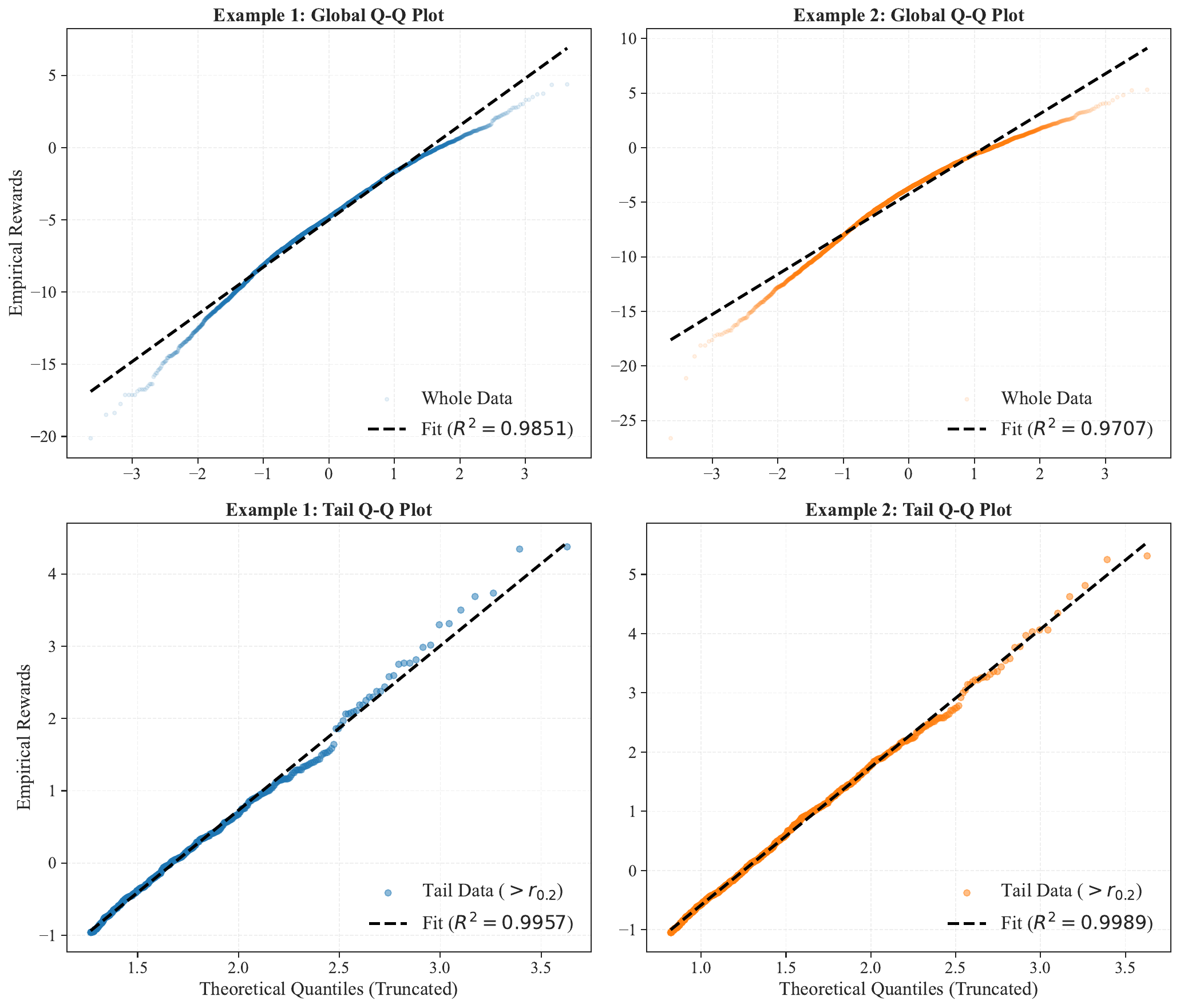}
    \caption{\textbf{Detailed Q-Q Plot Analysis.} 
    \textbf{Top Row:} Q-Q plots for the entire reward distribution against a standard Normal. The high density of points and slight deviations at the ends indicate body irregularities.
    \textbf{Bottom Row:} Q-Q plots for the tail region ($> R_{0.2}$) against a Truncated Normal. The strictly linear alignment and near-perfect $R^2$ scores confirm the validity of the Gaussian Tail Assumption for extrapolation.}
    \label{fig:app-qq-plots}
\end{figure}

 \begin{table}[h]
    \centering
    \caption{\textbf{Comparison of Goodness-of-Fit ($R^2$).}
    For the reward distributions shown in Figure~\ref{fig:reward-tail-plot},
    we compare the fit quality of a standard Gaussian on the entire dataset versus a Truncated Gaussian on the top 20\% tail. The tail fit consistently outperforms the global fit, indicating that Gaussian properties are most dominant in the extreme tail region.}
    \label{tab:r2-comparison}
    \begin{small}
    \begin{tabular}{lcc}
        \toprule
        & \textbf{Global Fit $R^2$}  & \textbf{Tail Fit $R^2$} \\
        \midrule
        Example 1 & 0.9851 & \textbf{0.9957} \\
        Example 2 & 0.9707 & \textbf{0.9989} \\
        \bottomrule
    \end{tabular}
    \end{small}
\end{table}

\subsection{Broader Empirical Validation}
To ensure that the Gaussian tail property is not unique to the two selected examples above, we extend our analysis to the reward distributions of 9 additional states.  
These states are generated by \textit{Llama-3.2-1B-Instruct} on randomly sampled questions from the AIME 2024 dataset, scored by \texttt{Skywork-Reward-V2-Llama-3.1-8B}.

Figure~\ref{fig:appendix_gaussian_check} displays the reward density histograms for these distributions. In each plot, we visualize the fit of a Gaussian curve (black dashed line) specifically to the top 20\% tail (blue bars), while the remaining body of the distribution is shown in grey.

As observed in Figure~\ref{fig:appendix_gaussian_check}, the distribution bodies are highly non-Gaussian, exhibiting significant skewness and multi-modality. In contrast, the tail regions consistently align with the fitted Gaussian envelope across all 9 examples. This confirms that despite the global complexity of the reward landscape, the critical tail behavior remains predictable, strictly validating our Gaussian Tail Assumption (Assumption~\ref{assum:gaussian-tail-model}).
\begin{figure}[h]
    \centering
    \includegraphics[width=1.0\textwidth]{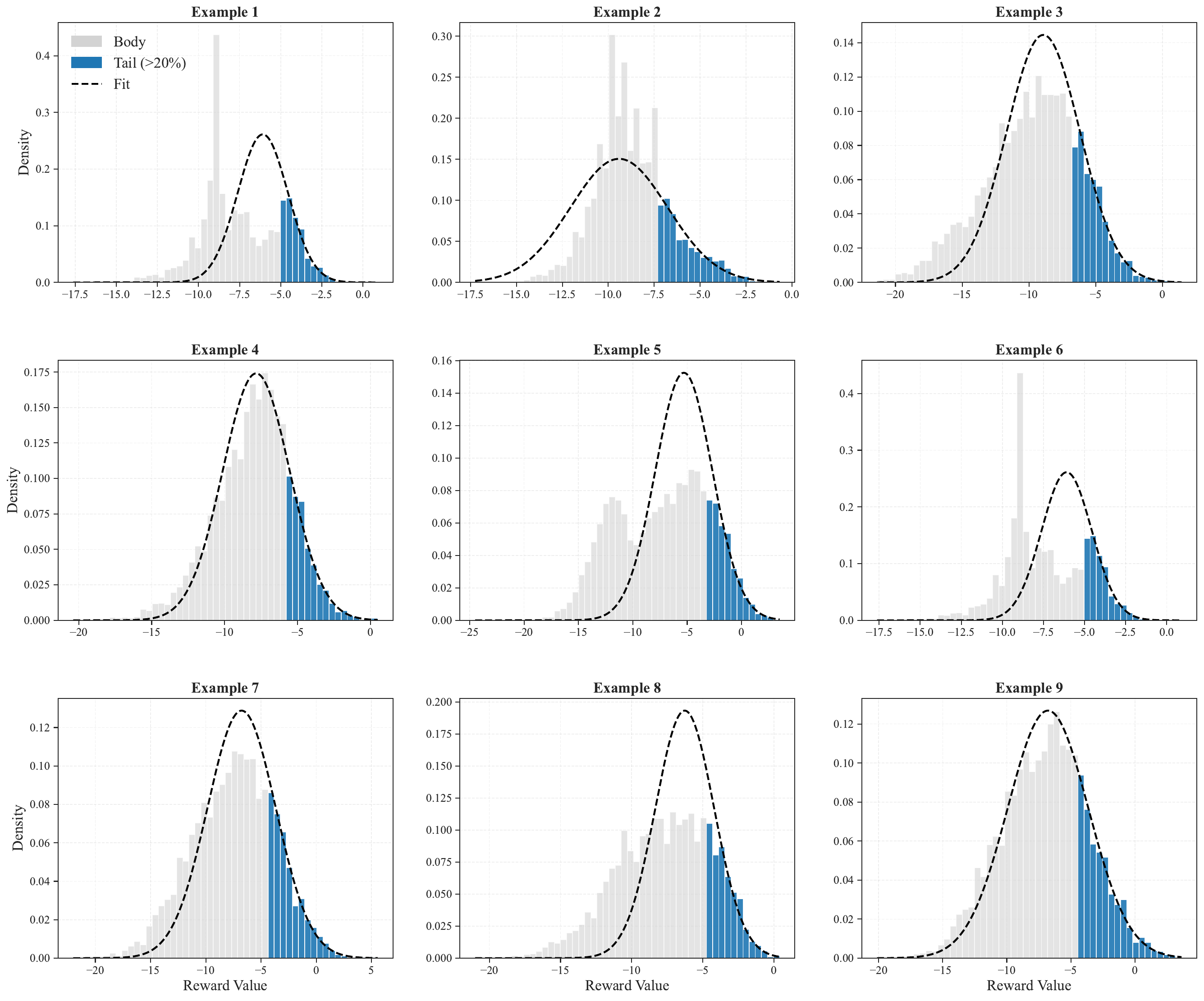}
    \caption{\textbf{Broad Visual Validation of Gaussian Tails.} 
    We visualize the reward densities for 9 randomly sampled query instances distinct from the main text examples.
    \textbf{Grey Bars:} The distribution body ($< 80$th percentile). 
    \textbf{Blue Bars:} The tail region ($> 80$th percentile). 
    \textbf{Dashed Line:} A Gaussian distribution fitted specifically to the tail statistics. 
    Despite the irregular and diverse shapes of the distribution bodies, the tail regions consistently follow a predictable Gaussian profile.}
    \label{fig:appendix_gaussian_check}
\end{figure}

\section{Experimental details}
\label{appendix:experimental_details}

We list the exact hyperparameter values used in our main experiments (Section~\ref{subsec:main_results}) in Table~\ref{tab:km_schedule}.

The search width $K$ and estimation budget $m$ are calculated dynamically based on the total budget $N$. 
We first determine the estimation budget $m$ by scaling with $\ln N$ and rounding to the nearest multiple of 5:
\begin{equation*}
    m(N) = 5 \times \operatorname{round}\left( \frac{(\ln N)^3}{25} \right).
\end{equation*}
Given $m(N)$, we calculate the target exploration width $K$. To ensure the total computational cost remains within bounds, we define $K$ as the target ratio clipped by the maximum feasible budget:
\begin{equation*}
    K(N) = \min \left( \operatorname{round}\left( \frac{N}{2m(N)} \right), \; \left\lfloor \frac{N}{m(N) + 2} \right\rfloor \right).
\end{equation*}
This schedule balances the trade-off between accurate estimation (increasing $m$) and sufficient exploration width ($K$), and aligns with the theoretical insights from Corollary~\ref{cor:regret_to_N}.

\begin{table}[h]
    \centering
    \caption{\textbf{Hyperparameter Schedule.} The exact estimation budget ($m$) and exploration width ($K$) used for each total budget ($N$). Calculated via the formulas: $m \approx \frac{1}{5}(\ln N)^3$ and $K \approx \frac{N}{2m}$.}
    \label{tab:km_schedule}
    \begin{tabular}{ccc}
        \toprule
        \textbf{Total Budget} ($N$) & \textbf{Estimation Budget} ($m$) & \textbf{Exploration Width} ($K$) \\
        \midrule
        100  & 20 & 2 \\  
        200  & 30 & 3 \\ 
        300  & 35 & 4 \\  
        500  & 50 & 5 \\ 
        700  & 55 & 6 \\  
        800  & 60 & 7 \\  
        1000 & 65 & 8 \\ 
        \bottomrule
    \end{tabular}
\end{table}

\section{Proof of Theorem~\ref{theorem:error-bound}}
\label{appendix:proof_error_bound}

By Assumption~\ref{assum:gaussian-tail-model}, we know the density
of the reward distribution $p(r) = \phi(r; \mu, \sigma)$ for $r \geq r_{\alpha/2}$.
We denote $M(N)$ as the expected maximal value when sampling $N$ times from $\cN(\mu, \sigma^2)$, then we know:

\begin{equation}
  \label{eq:def_M_N}
  M(N) = \mu + \sigma E(N), \quad \text{where} \quad E(N) = \E_{Z_i \sim \cN(0,1)} \left[ \max_{i=1}^N Z_i \right]
\end{equation}

and $\hat{V}_N(s) = \hat{\mu} + \hat{\sigma} E(N)$ is an estimator of $M(N)$. 

And we can split the total error into two parts:

\begin{align*}
  \left| V_N(s) - \hat{V}_N(s) \right| \leq \left| V_N(s) - M(N) \right| +
  \left| \hat{V}_N(s) - M(N) \right|
\end{align*}

where the first part is the bias error induced by
the Gaussian tail approximation, and the second part
is the estimation error of the Gaussian maximal.

\subsection{Estimation Error}

We notice that, under the Gaussian Tail Model (Assumption~\ref{assum:gaussian-tail-model}), we can 
write the estimation error as:

\begin{align*}
  \left| \hat{V}_N - M(N) \right| 
  & = \left| (\hat{\mu} - \mu) + (\hat{\sigma} - \sigma) E(N) \right| \\
  & \leq |\hat{\mu} - \mu| + |\hat{\sigma} - \sigma| E(N)
\end{align*}

So it suffices to bound the estimation errors $|\hat{\mu} - \mu|$ and $|\hat{\sigma} - \sigma|$, or equivalently 
$| \hat{\mu}_{\text{tail}} - \E[R | R \geq r_{\alpha/2}] |$ and $| \hat{\sigma}_{\text{tail}} - \sqrt{\Var[R | R \geq r_{\alpha/2}]} |$ 
by Equation~\eqref{eq:param-inversion},
where $\hat{\mu}_{\text{tail}}$ and $\hat{\sigma}_{\text{tail}}$ 
are the empirical mean and standard deviation
computed from the tail dataset $D_{\text{tail}} = \{ r_i \mid r_i \geq \hat{r}_{\alpha/2}, i = 1, 2, \cdots, m_{tail} \}$.

We can decompose these errors into two components:

\begin{itemize}
  \item \textbf{Bias error:} 
  This error arises from approximating the true threshold $r_{\alpha/2}$ with the empirical threshold $\hat{r}_{\alpha/2}$. 
  We denote:
  \begin{align*}
    \mu(r) = \E[R | R \geq r], \quad \sigma^2(r) = \Var[R | R \geq r]
  \end{align*}
  as the conditional mean and variance given a generic threshold $R$. The bias error can then be expressed as:
  \begin{align*}
    |\mu(\hat{r}_{\alpha/2}) - \mu(r_{\alpha/2})|, \quad |\sigma(\hat{r}_{\alpha/2}) - \sigma(r_{\alpha/2})|
  \end{align*}
  
\item \textbf{Statistical error:} This error stems from estimating the conditional moments $\mu(\hat{r}_{\alpha/2})$ and $\sigma^2(\hat{r}_{\alpha/2})$ using a finite number of tail samples,
and can be expressed as:
  \begin{align*}
    |\hat{\mu}_{\text{tail}} - \mu(\hat{r}_{\alpha/2})|, \quad |\hat{\sigma}_{\text{tail}} - \sigma(\hat{r}_{\alpha/2})|
  \end{align*}
\end{itemize}

We first construct a high-probability bound on the error of the empirical threshold $\hat{r}_{\alpha/2}$:

\begin{lemma}\label{lemma:threshold-error}
  Under Assumptions \ref{assum:gaussian-tail-model}, there exists a constant $c_1 > 0$ such that for all sample sizes $m$ satisfying $m \geq c_1 \log(1/\delta)$, with probability at least $1 - \delta$:
  
  $$ \left|\hat{r}_{\alpha/2} - r_{\alpha/2}\right| \leq c_2 \sqrt{\frac{\log(1/\delta)}{m}}, $$
  
  where $c_1, c_2$ are generic positive constants depending on $(\alpha, \sigma)$.
\end{lemma}

\begin{proof}[Proof of Lemma~\ref{lemma:threshold-error}]
  Let $\hat{F}_m(r)$ denote the empirical CDF of the reward samples $\{r_i\}_{i=1}^m$. By the Dvoretzky–Kiefer–Wolfowitz (DKW) inequality, the event
  $$ \mathcal{E} = \left\{ \sup_{r \in \mathbb{R}} | \hat{F}_m(r) - F(r) | \leq \epsilon_m \right\}, \quad \text{where} \quad \epsilon_m = \sqrt{\frac{\log(2/\delta)}{2m}}, $$
  occurs with probability at least $1 - \delta$. We condition on $\mathcal{E}$ for the remainder of the proof.

  Denote $\tau = r_{\alpha/2} - r_{\alpha}$ where $r_\alpha = \Phi^{-1}(1 - \alpha; \mu, \sigma)$ is the $\alpha$-quantile of the Gaussian distribution. 
  We first establish that $\hat{r}_{\alpha/2}$ lies within the local neighborhood $\mathcal{N} = [r_{\alpha/2} - \tau, r_{\alpha/2} + \tau]$. 
  By Assumption~\ref{assum:gaussian-tail-model}, the reward density $p(r)$ equals the Gaussian density $\phi(r; \mu, \sigma)$ in this region, and thus $p(r) \geq \kappa >0$  for all 
  $r \in \mathcal{N}$, where $\kappa = \frac{1}{\sqrt{2\pi}\sigma} e^{-\frac{(r_{\alpha/2} + \tau - \mu)^2}{2\sigma^2}} $ is a constant depending on $(\alpha, \sigma)$. 
  \footnote{Since $\tau = \sigma(\Phi^{-1}(1 - \alpha/2) - \Phi^{-1}(1 - \alpha))$ and $r_{\alpha/2} - \mu = \sigma \Phi^{-1}(1 - \alpha/2)$, the constant $\kappa$ is independent of $\mu$.}
  
  When $m \geq \frac{4}{\kappa^2\tau^2} \log(1/\delta)$, we have:

  $$ \epsilon_m = \sqrt{\frac{\log(2/\delta)}{2m}} \leq \frac{1}{2} \kappa \tau. $$

  Under this condition, and utilizing the density lower bound $F(r_{\alpha/2} + \tau) \geq (1-\alpha/2) + \kappa \tau$, we have:
  \begin{align*}
      \hat{F}_m(r_{\alpha/2} + \tau) &\geq F(r_{\alpha/2} + \tau) - \epsilon_m 
      \geq (1-\alpha/2) + \kappa \tau - \frac{1}{2}\kappa \tau 
      > 1 - \alpha/2 + \frac{1}{m}.
  \end{align*}
 
  This implies $\hat{r}_{\alpha/2} \leq r_{\alpha/2} + \tau$. A symmetric argument shows $\hat{R}_\alpha \geq R_\alpha - \tau$. Thus, $\hat{r}_{\alpha/2} \in \mathcal{N}$.

  Since $\hat{r}_{\alpha/2} \in \mathcal{N}$, we can invert the CDF difference using the density lower bound $\kappa$. By the Mean Value Theorem on the CDF $F(\cdot)$:
  $$ |F(\hat{r}_{\alpha/2}) - F(r_{\alpha/2})| = p(\tilde{r}) |\hat{r}_{\alpha/2} - r_{\alpha/2}| \geq \kappa  |\hat{r}_{\alpha/2} - r_{\alpha/2}|. $$
  
  Rearranging this and relating the true CDF to the empirical quantile via the DKW bound:
  \begin{align*}
       |\hat{r}_{\alpha/2} - r_{\alpha/2}| &\leq \frac{1}{\kappa} |F(\hat{r}_{\alpha/2}) - F(r_{\alpha/2})| \\
      &\leq \frac{1}{\kappa} \left( |F(\hat{r}_{\alpha/2}) - \hat{F}_m(\hat{r}_{\alpha/2})| + |\hat{F}_m(\hat{r}_{\alpha/2}) - (1-\alpha/2)| \right) \\
      &\leq \frac{1}{\kappa} \left( \epsilon_m + \frac{1}{m} \right).
  \end{align*}
  
  Substituting $\epsilon_m = \sqrt{\frac{\log(2/\delta)}{2m}}$, we know for $m \geq c_1 \log(1/\delta)$, it holds with probability at least $1 - \delta$ that:
  $$ |\hat{r}_{\alpha/2} - r_{\alpha/2}| \leq c_2 \sqrt{\frac{\log(1/\delta)}{m}}, $$
  for some constants $c_1, c_2 > 0$ depending on $(\alpha, \sigma)$, which completes the proof.
\end{proof}

Then for the bias error, we have the following lemma:

\begin{lemma}\label{lemma:bias-error}
  Under Assumptions \ref{assum:gaussian-tail-model}, there exist constants $\tau > 0$ and $c > 0$ (depending on $\alpha, \sigma$) such that for all $r$ satisfying $\Delta = |r - r_{\alpha/2}| \leq \tau$, we have:
  
  $$ \left| \mu(r) - \mu(r_{\alpha/2}) \right| \leq c \Delta, \quad \text{and} \quad \left| \sigma(r) - \sigma(r_{\alpha/2}) \right| \leq c \Delta. $$
  
\end{lemma}

\begin{proof}[Proof of Lemma~\ref{lemma:bias-error}]
  We choose $\tau = r_{\alpha/2} - r_{\alpha} = \sigma \left( \Phi^{-1}(1 - \alpha/2) - \Phi^{-1}(1 - \alpha) \right)$, 
  and denote $\lambda(x) = \frac{\phi(x)}{1 - \Phi(x)}$ as the inverse Mills ratio. 

  We prove the two claims separately:
\begin{enumerate}
    \item \textbf{Bias error of mean} 
    
    Let $\mu_{G}(r)$ denote the mean of the Gaussian distribution truncated at $r$:
    $$ \mu_{G}(r) = \frac{\int_{r}^{\infty} x \phi(x;\mu,\sigma^2)dx}{\int_{r}^{\infty} \phi(x;\mu,\sigma^2)dx}. $$

    For any $r$ satisfying $\Delta = |r - r_{\alpha/2}| \leq \tau$, we know that both $r$ and $r_{\alpha/2}$ lie in the region where the reward density $p(r)$ equals the Gaussian density $\phi(r;\mu,\sigma^2)$ by 
    Assumption~\ref{assum:gaussian-tail-model}. Thus, we have:
    \begin{align*}
    \left| \mu(r) - \mu(r_{\alpha/2}) \right|  = \left| \mu_{G}(r) - \mu_{G}(r_{\alpha/2}) \right|
    \end{align*}
  
    This implies that it's equivalent to bound the difference between the mean of truncated Gaussian at two truncation points, and we know that the mean of the truncated Gaussian is given by $\mu + \sigma \lambda(\frac{r-\mu}{\sigma})$.
  A well-known property of $\lambda(x)$ is that $ 0 \le \lambda'(x) \le 1$ for all $x \in \mathbb{R}$ \citep{sampford1953some}, by which we have:
    \begin{align*}
        \left| \mu_{G}(r) - \mu_{G}(r_{\alpha/2}) \right| &\leq \left| \sigma \lambda(\frac{r - \mu}{\sigma}) - \sigma \lambda(\frac{r_{\alpha/2} - \mu}{\sigma}) \right| \\
        & \leq \sup_{x} |f'(x)| \cdot \sigma \left|\frac{r_{\alpha/2} - \mu}{\sigma} - \frac{r - \mu}{\sigma} \right| \leq \Delta.
    \end{align*}

    This completes the proof of the first claim.

\item \textbf{Bias error of variance}

    Let $\sigma^2_G(r)$ denote the variance of the Gaussian distribution truncated at $r$. Again, for any $r$ satisfying $\Delta = |r - r_{\alpha/2}| \leq \tau$, we have:
    \begin{align*}
        \left| \sigma^2(r) - \sigma^2(r_{\alpha/2}) \right| = \left| \sigma^2_G(r) - \sigma^2_G(r_{\alpha/2}) \right|.
    \end{align*}
    
    We know that the variance of the truncated Gaussian can be expressed as:
    $$ \sigma^2_G(r) = \sigma^2 \delta(\frac{r - \mu}{\sigma}), \quad \text{where} \quad \delta(z) =  1 + z \lambda(z) - \lambda(z)^2. $$
    use the properties of the standard truncated normal variance $\delta(z) = 1 + z \lambda(z) - \lambda(z)^2$. 
    Direct calculation yields that 
     $|\delta'(z)| \leq 1$ for all $z$, we have:

    \begin{align*}
       \left| \sigma^2_G(r) - \sigma^2_G(r_{\alpha/2}) \right|  = \sigma^2 \left| \delta(\frac{r - \mu}{\sigma}) - \delta(\frac{r_{\alpha/2}-\mu}{\sigma}) \right| \leq \sigma^2 \left| \frac{r_{\alpha/2}-\mu}{\sigma} - \frac{r-\mu}{\sigma} \right| \leq \sigma \Delta .
    \end{align*}

    Since $\sigma(r) > 0$ and $\sigma(r_{\alpha/2})$ is a positive constant depending only on $\sigma$ and $\alpha$, we have:
    \begin{align*}
        \left| \sigma(r) - \sigma(r_{\alpha/2}) \right| \leq \frac{1}{\sigma(r_{\alpha/2})} \left| \sigma^2(r) - \sigma^2(r_{\alpha/2}) \right| \leq c \Delta,
    \end{align*}
    which completes the proof of the second claim.

\end{enumerate}

\end{proof}

And for the statistical error, we have the following concentration result:

\begin{lemma}\label{lemma:statistical-error}
  Under Assumption \ref{assum:gaussian-tail-model}, for any $\delta \in (0,1)$, when the sample size $m \geq c_1 \log(1/\delta)$,
  then with probability at least $1 - \delta$,
  \begin{align*}
    \left|\hat{\mu}_{tail} - \mu(\hat{r}_{\alpha/2}) \right| & \leq  c_2 \sqrt{\frac{\log(1/\delta)}{m}} \\
    \left|\hat{\sigma}_{tail} - \sigma(\hat{r}_{\alpha/2}) \right| & \leq c_2 \sqrt{\frac{\log(1/\delta)}{m}}
  \end{align*}

  where $c_1, c_2$ are constants only depending on $(\alpha, \sigma)$.
\end{lemma}

\begin{proof}[Proof of Lemma~\ref{lemma:statistical-error}]
  We first define the event $E_R = \left\{ \left| r_{\alpha/2} - \hat{r}_{\alpha/2} \right| \leq r_{\alpha/2} - r_\alpha \right\}$. Under this event,  we know by 
  Assumption~\ref{assum:gaussian-tail-model} $R| R \geq \hat{r}_{\alpha/2}$ is a truncated Gaussian random variable, and thus is sub-Gaussian with parameter $\sigma$. And by 
  Lemma~\ref{lemma:threshold-error}, when $m \geq c \log(1/\delta)$ for some constant $c_1$ depending only on $\alpha$ and $\sigma$, we have $\mathbb{P}(E_R) \geq 1 - \delta/3$.

  We further define the following two events for the concentration of the empirical mean and variance:
  \begin{align*}
    E_{\mu} & = \left\{ \left|\hat{\mu}_{tail} - \mu(\hat{r}_{\alpha/2}) \right| \leq  \sqrt{2}\sigma\sqrt{\frac{\log(6/\delta)}{\lfloor m\alpha/2 \rfloor}} \right\},\\
    E_{\sigma} & = \left\{ \left|\sigma^2_{tail} - \sigma^2(\hat{r}_{\alpha/2}) \right| \leq 2 \sqrt{2}\sigma^2 \sqrt{\frac{\log(6/\delta)}{\lfloor m\alpha/2 \rfloor}} \right\}
  \end{align*}
  
  here $\sigma^2_{tail} = \frac{1}{\lfloor m\alpha/2 \rfloor} \sum_{r_i \in \D_{tail}} (r_i - \mu(\D_{tail}))^2$ is the 
  true variance of the tail dataset $\D_{tail}$. By concentration inequality of sub-Gaussian and 
  sub-exponential random variables, we know that under the event $E_R$, when $m \geq c' \log(1/\delta)$ for some constant $c'$ depending only on $\alpha$ and $\sigma$, we have $\mathbb{P}(E_\mu) \geq 1 - \delta/3$ and $\mathbb{P}(E_\sigma) \geq 1 - \delta/3$.

  So if we define the event $E = E_\mu \cap E_\sigma \cap E_R$, then by union bound, we know that when $m \geq c_1 \log(1/\delta)$ for some constant $c_1$ depending only on $\alpha$ and $\sigma$
   , we have $\mathbb{P}(E) \geq 1 - \delta$. Under the event $E$,
  we directly obtain:

  $$\left|\hat{\mu}_{tail} - \mu(\hat{r}_{\alpha/2}) \right| \leq \sqrt{2}\sigma\sqrt{\frac{\log(6/\delta)}{\lfloor m\alpha/2 \rfloor}} \leq c_2 \sqrt{\frac{\log(1/\delta)}{m}}$$

  for some constant $c_2$ depending only on $\alpha$ and $\sigma$. And for the variance estimation, conditioned on the event $E$, we have:

  \begin{align*}
    \left| \hat{\sigma}^2_{tail} - \sigma^2(\hat{r}_{\alpha/2}) \right| & = \left| \frac{1}{\lfloor m\alpha/2 \rfloor} \sum_{r_i \in \D_{tail}} (r_i - \hat{\mu}(\D_{tail}))^2 - \sigma^2(\hat{r}_{\alpha/2}) \right| \\
    & \leq \left| \frac{1}{\lfloor m\alpha/2 \rfloor} \sum_{r_i \in \D_{tail}} (r_i - \mu(\hat{r}_{\alpha/2}))^2 - \sigma^2(\hat{r}_{\alpha/2}) \right| + \left| \hat{\mu}(\D_{tail}) - \mu(\hat{r}_{\alpha/2}) \right|^2 \\  
    & \leq 2 \sqrt{2}\sigma^2 \sqrt{\frac{\log(6/\delta)}{\lfloor m\alpha/2 \rfloor}} + 2 \sigma^2 \frac{\log(6/\delta)}{\lfloor m\alpha/2 \rfloor} \\
    & \leq c_2' \sqrt{\frac{\log (1/\delta)}{m}}
  \end{align*}

  And since $\sigma(\hat{r}_{\alpha/2}) \geq c_\sigma > 0$ for some constant $c_\sigma$ depending only on $\sigma$ and $\alpha$ conditioned on the event $E_R$, we have:
  
  $$ \left| \hat{\sigma}_{tail} - \sigma(\hat{r}_{\alpha}) \right| \leq \frac{1}{\sigma(\hat{r}_{\alpha/2})} \left| \hat{\sigma}^2_{tail} - \sigma^2(\hat{r}_{\alpha}) \right| \leq c_2 \sqrt{\frac{\log (1/\delta)}{m}} $$

  for some constant $c_2$ depending only on $\alpha$ and $\sigma$. 
  This completes the proof.
\end{proof}

Concluding these lemmas, we have the following proposition on the overall estimation error of $\hat{V}_N(s)$:

\begin{proposition}\label{proposition: estimation-error-of-VN}
  Under Assumptions \ref{assum:gaussian-tail-model}, 
  for any $\delta \in (0,\frac{1}{2})$, when the sample size $m \geq c_1 \log(1/\delta)$,
  then with probability at least $1 - \delta$,
    $$ \left|\hat{V}_N(s) - M(N) \right| \leq c_2 \sqrt{\log N} \sqrt{\frac{\log(1/\delta)}{m}} $$

  where $c_1, c_2$ are constants only depending on $(\alpha, \sigma)$.
\end{proposition}

\begin{proof}[Proof of Proposition~\ref{proposition: estimation-error-of-VN}]
  We denote the same event as in the proof of Lemma~\ref{lemma:statistical-error}
  \begin{align*}
    E = E_R \cap E_{\mu} \cap E_{\sigma} 
  \end{align*}

  where $E_R, E_\mu, E_\sigma$ are defined in the proof of Lemma~\ref{lemma:statistical-error}, and we know that when
  $m \geq c \log(1/\delta)$, $\mathbb{P}(E) \geq 1 - \delta$, for some constant $c > 0$ depending on $(\alpha, \sigma)$.

  And under the event $E$, by Lemma \ref{lemma:bias-error}, when $m \geq c' \log(1/\delta)$ for some constant $c' > 0$ depending on $(\alpha, \sigma)$, we have:
  \begin{align*}
    |\mu(\hat{r}_{\alpha/2}) - \mu(r_{\alpha/2})| & \leq c_3  \left|\hat{r}_{\alpha/2} - r_{\alpha/2}\right| \leq c_3' \sqrt{\frac{\log(1/\delta)}{m}} \\
    |\sigma(\hat{r}_\alpha/2) - \sigma(r_{\alpha/2})| & \leq c_3 \left|\hat{r}_{\alpha/2} - r_{\alpha/2}\right| \leq  c_3' \sqrt{\frac{\log(1/\delta)}{m}}
  \end{align*}

  for some constants $c_3, c_3' > 0$ depending on $(\alpha, \sigma)$.

  Therefore, by Lemma \ref{lemma:statistical-error} and the triangle inequality, we know when $m \geq c_1 \log(1/\delta)$ for some constant $c_1 > 0$ depending on $(\alpha, \sigma)$
  , with probability at least $1 - \delta$,
  \begin{align*}
    |\hat{\mu}_{tail} - \mu(r_{\alpha/2})| & \leq  \left| \hat{\mu}_{tail} - \mu(\hat{r}_{\alpha/2}) \right| + \left| \mu(\hat{r}_{\alpha/2}) - \mu(r_{\alpha/2}) \right| \leq c_\mu \sqrt{\frac{\log(1/\delta)}{m}}\\
    |\hat{\sigma}_{tail} - \sigma(r_{\alpha/2})| & \leq  \left| \hat{\sigma}_{tail} - \sigma(\hat{r}_{\alpha/2}) \right| + \left| \sigma(\hat{r}_{\alpha/2}) - \sigma(r_{\alpha/2}) \right| \leq c_\sigma \sqrt{\frac{\log(1/\delta)}{m}}
  \end{align*}
  
  and by the formula of $\hat{\mu}$, $\hat{\sigma}$ in Algorithm~\ref{alg:tail-mom}
  , for some constants $c_\mu', c_\sigma' > 0$ depending on $(\alpha, \sigma)$, we have:
  \begin{align*}
   \left| \hat{\sigma} - \sigma \right| & = \frac{|\hat{\sigma}(\D_{tail}) - \sigma(r_{\alpha/2})|}{\sqrt{\delta(\frac{r_{\alpha/2 - \mu }}{\sigma})}} \leq c_\sigma' \sqrt{\frac{\log(1/\delta)}{m}} \\
  \left| \hat{\mu} - \mu \right| & = |\hat{\mu}(\D_{tail}) - \mu(r_{\alpha/2})| + \lambda(\frac{r_{\alpha} - \mu}{\sigma}) |\hat{\sigma} - \sigma|  \leq c_\mu' \sqrt{\frac{\log(1/\delta)}{m}}
  \end{align*}

  where $\lambda(z)$ is the inverse Mills ratio, and $\delta(z) = 1 + z \lambda(z) - \lambda(z)^2$.

  Finally, we can bound the overall estimation error of $\hat{V}_N$ when $m \geq c_1 \log(1/\delta)$ and with probability at least $1 - \delta$:

  \begin{align*}
    |\hat{V}_N(s) - M(N)| & = |\hat{\mu} + \hat{\sigma} E(N) - \mu - \sigma E(N)| \\
    & \leq |\hat{\mu} - \mu| + | \hat{\sigma} - \sigma| \cdot E(N) \\
    & \leq c_\mu' \sqrt{\frac{\log(1/\delta)}{m}} + c_\sigma' \sqrt{\frac{\log(1/\delta)}{m}} \sqrt{2\log N} \\
    & \leq c_2 \sqrt{ \log (N)} \sqrt{\frac{\log(1/\delta)}{m}}
  \end{align*}
  
  for some constant $c_2 > 0$ depending only on $(\alpha, \sigma)$. This completes the proof.
\end{proof}

\subsection{Bias Error}

We can establish a lemma bounding the approximation bias $|V_N(s) - M(N)|$. This bias arises because the true reward distribution $F_s$ differs from the Gaussian distribution in the "body" of the distribution ($r < r_{\alpha/2}$).

\begin{lemma}
  \label{lemma:bias-error-of-VN}Under Assumption~\ref{assum:gaussian-tail-model}, the difference between the true expected maximum $V_N(s)$ and the Gaussian expected maximum $M(N)$ satisfies:
  
  \begin{equation}
    |V_N(s) - M(N)| \leq c (1-\alpha/2)^{N},
  \end{equation}
  
  where $c$ is a finite constant depending only on $(\alpha, \sigma, C_R)$.
\end{lemma}

\begin{proof}[Proof of Lemma~\ref{lemma:bias-error-of-VN}]
  
  Let $F(r)$ be the CDF of the true reward distribution and $G(r) = \Phi(r; \mu, \sigma^2)$ be the CDF of the Gaussian approximation. And we have:
  
  \begin{equation*}
    V_N(s) - M(N) = \int_{-\infty}^{\infty} (G(r)^N - F(r)^N) dr.
  \end{equation*}
  
  By Assumption~\ref{assum:gaussian-tail-model}, the CDFs are identical for $r \geq r_{\alpha/2}$:
  
  $$    F(r) = G(r), \quad \forall r \geq r_{\alpha/2},$$
  
  which implies that the integral over $[r_{\alpha/2}, \infty)$ is zero. Therefore, we only need to consider the integral over $(-\infty, r_{\alpha/2}]$:
  
  \begin{equation}
    |V_N(s) - M(N)| = \left| \int_{-\infty}^{r_{\alpha/2}} (G(r)^N - F(r)^N) dr \right| \leq \int_{-\infty}^{r_{\alpha/2}} |G(r)^N - F(r)^N| dr.
  \end{equation}
  
  For $r \leq r_{\alpha/2}$, both CDFs are bounded by $1-\alpha/2$. Then we can further bound the integral:

  \begin{align*}
    \int_{-\infty}^{r_{\alpha/2}} |G(r)^N - F(r)^N| dr & \leq \int_{-\infty}^{r_{\alpha/2}} (G(r)^N + F(r)^N) dr \\
    & \leq (1-\alpha/2)^{N-1} \int_{-\infty}^{r_{\alpha/2}} (G(r) + F(r)) dr.
  \end{align*}

  Let $c = \int_{-\infty}^{r_{\alpha/2}} (G(r) + F(r)) dr$. By the change of variable $u = r - \mu$, we observe that the integral depends only on the centered distributions and the centered threshold $r_{\alpha/2} - \mu$. 
  Thus, $c$ is independent of the location parameter $\mu$. Finally, Assumption~\ref{assum:gaussian-tail-model} guarantees finite second moments ($\mathbb{E}[R^2] < C_R$),
   which implies $\mathbb{E}[|R|] < \infty$ and ensures the integral converges to a finite constant which only depends on $C_R$ and $\sigma$. This completes the proof.
\end{proof}

\subsection{Proof of Theorem~\ref{theorem:error-bound}}

\begin{proof}[Proof of Theorem~\ref{theorem:error-bound}]
  By the triangle inequality, we decompose the total error into the bias error and the estimation error:
  
  \begin{align*}
    |\hat{V}_N(s) - V_N(s)| & \leq |\hat{V}_N(s) - M(N)| + |M(N) - V_N(s)| 
  \end{align*}
  
  The first term is bounded by Proposition~\ref{proposition: estimation-error-of-VN}, and the second term is bounded by Lemma~\ref{lemma:bias-error-of-VN}. Combining these results, we conclude that when $m \geq c_1 \log(1/\delta)$, with probability at least $1 - \delta$,
  
  \begin{align*}
    |\hat{V}_N(s) - V_N(s)| & \leq c_2 \big(\sqrt{\log N} \sqrt{\frac{\log(1/\delta)}{m}} +  (1-\alpha/2)^{N} \big)
  \end{align*}
  
  This completes the proof.
  
\end{proof}

\newpage

\section{Proof of Theorem~\ref{thm:regret_bound}}
\label{appendix:proof-regret-bound}

We first consider the selection error of our algorithm. Let $K$ be smaller than $N$, and $\{s_1, \dots, s_K\}$ be $K$ candidate states with true return levels $\{V_N (s_i)\}_{i=1}^K$ and estimators $\{\hat{V}_N(s_i)\}_{i=1}^K$.
And we define the estimation error for each candidate as $\epsilon_i = \hat{V}_N (s_i) - V_N (s_i)$ for all $i \in [K]$.

By Assumption~\ref{assume:gaussian-tail}, we know for any candidate $s$, Theorem~\ref{theorem:error-bound} applies with some constants $c_1, c_2$ depending on $\sigma(s), \alpha$ and $C_R$. And since 
$\sigma(s) \in [\sigma_{min}, \sigma_{max}]$, there exist uniform constants $c_1, c_2$ depending only on $\sigma_{min}, \sigma_{max}, \alpha, C_R$ such that Theorem~\ref{theorem:error-bound} holds for all candidates.

We can then bound the expectation of the maximal estimation error among all candidates:

\begin{lemma}
\label{lemma:selection-regret}
  Under Assumption~\ref{assume:gaussian-tail}, 
  assume the sample size $m$ satisfies $m \geq c_0 \log K$ for a sufficiently large constant $c_0 > 0$ 
  . Then for sufficiently large $N$, we have:
  \begin{align*}
    \E\left[ \max_{i \in [K]} |\epsilon_i| \right] \leq 
    c_1 \sqrt{\frac{\log N \log K}{m}}
  \end{align*}

  where $c_1$ is a constant that depends on $\sigma_{min}, \sigma_{max}, \alpha, M_\mu, C_R.$
\end{lemma}

\begin{proof}
  First, we derive a tail probability bound from Theorem~\ref{theorem:error-bound}. The theorem states that for any $\delta$ satisfying the condition $m \geq c_1 \log(1/\delta)$ (equivalent to $\delta \geq e^{-m/c_1}$), 
  we have with probability at least $1-\delta$:
  
  $$|\epsilon_i| \leq c_2 \big( \sqrt{\log N} \sqrt{\frac{\log(1/\delta)}{m}}+ (1-\alpha/2)^{N} \big).$$

  For sufficiently large $N$, the bias term $(1-\alpha/2)^{N}$ is dominated by the first term, which means there exists a new constant $c_2'$ such that:

  $$|\epsilon_i| \leq c_2' \sqrt{\log N} \sqrt{\frac{\log(1/\delta)}{m}}.$$

  Let $t$ be a positive threshold, then we know that when $t \leq t_{max}$, the following tail bound holds:

  $$ \mathbb{P}(|\epsilon_i| > t) \leq \exp\left( - \frac{m t^2}{c_2^{'2} \log N} \right) $$

  where $t_{max} = c_2' \sqrt{\frac{\log N}{c_1}}$.

  We define the good event region based on this maximum valid threshold. Let $E$ be the event where the error for all candidates is within the valid range of the theorem:
  
  $$E = \left\{ \max_{i \in [K]} |\epsilon_i| \leq t_{max} \right\}$$
  
  Note that the complement event $E^C$ corresponds to the case where the error exceeds the range covered by the sample size condition in Theorem~\ref{theorem:error-bound}. 
  The probability of this occurring for a single arm is at most $\delta_{min} = e^{-m/c_1}$. By a union bound:
  
  $$\mathbb{P}(E^C) \leq K e^{-m/c_1}$$
  
  We decompose the expectation:
  
  \begin{align*}
    \E\left[ \max_{i \in [K]} |\epsilon_i| \right]& = \E\left[ \max_{i \in [K]} |\epsilon_i| \cdot \Ind(E) \right] + \E\left[ \max_{i \in [K]} |\epsilon_i| \cdot \Ind(E^C)\right]
  \end{align*}
  
  \textbf{Term 1: Expectation on the Good Event.} For the first term, we use the integral identity for expectation. 
  Let $M = \max_{i \in [K]} |\epsilon_i|$.
  
  \begin{align*}
    \E[ M \cdot \Ind(E) ] &= \int_{0}^{t_{max}} \mathbb{P}(M > t) dt \\ 
    &\leq \int_{0}^{t_{max}} \min\left(1, \sum_{i=1}^K \mathbb{P}(|\epsilon_i| > t) \right) dt
  \end{align*}
  
  Let $\lambda = \frac{m}{c_2^{'2} \log N}$. Then:
  
  \begin{align*}
    \E[ M \cdot \Ind(E) ] &\leq \int_{0}^{\infty} \min\left(1, K e^{-\lambda t^2} \right) dt
  \end{align*}
  
  We split the integral at $t_0 = \sqrt{\frac{\log K}{\lambda}}$. For $t \leq t_0$, the integrand is bounded by 1. For $t > t_0$, it is bounded by $K e^{-\lambda t^2}$.
  
  \begin{align*}
    \int_{0}^{\infty} \min\left(1, K e^{-\lambda t^2} \right) dt&= t_0 + \int_{t_0}^\infty K e^{-\lambda t^2} dt \\ 
    &\leq \sqrt{\frac{\log K}{\lambda}} + \frac{K}{\lambda t_0} e^{-\lambda t_0^2} \times \frac{1}{2} \quad \text{(using } \int_x^\infty e^{-t^2} dt \leq \frac{e^{-x^2}}{2x} \text{)} \\
    &= \sqrt{\frac{\log K}{\lambda}} + \frac{K}{\lambda \sqrt{\frac{\log K}{\lambda}}} \frac{1}{K} \cdot c \\ 
    &\leq c \sqrt{\frac{\log K}{\lambda}} = c' \sqrt{\frac{\log N \log K}{m}}
  \end{align*}

  for some constant $c, c'$ depending on $(\sigma_{min}, \sigma_{max}, \alpha, C_R)$.
  
  \textbf{Term 2: Expectation on the Bad Event.} For the second term, we use the Cauchy-Schwarz inequality:
  
  \begin{align*}
    \E\left[ \max_{i \in [K]} |\epsilon_i| \cdot \Ind(E^C)\right]& \leq \sqrt{\E\left[ \max_{i \in [K]} \epsilon_i^2 \right]} \sqrt{\mathbb{P}(E^C)}
  \end{align*}
  
  We have derived $\mathbb{P}(E^C) \leq K \exp(-m/c_1)$. So it suffices to bound $\E\left[ \max_{i \in [K]} \epsilon_i^2 \right]$. And we know:

  \begin{align*}
    \E\left[ \max_{i \in [K]} \epsilon_i^2 \right] \leq 2\E\left[ \max_{i \in [K]} V_N(X_i)^2 \right] + 2 \E\left[ \max_{i \in [K]} \hat{V}_N(X_i)^2 \right]
  \end{align*}

  For the first term, we have:
  \begin{align*}
    \E\left[ \max_{i \in [K]} V_N(X_i)^2 \right] & = \E\left[ \max_{i \in [K]} \left( \mu_i + \sigma_i E(N) \right)^2 \right] \\
    & \leq 2 \E\left[ \max_{i \in [K]} \mu_i^2 \right] + 2 E(N)^2 \sigma_{max}^2 \\
    & \leq 2K \E\left[ \mu_i^2 \right] + 2 \sigma_{max}^2 \log N \\
    & \leq c (K + \log N)
  \end{align*}
  for some constant $c$ depending on $(\sigma_{max}, M_\mu)$. Here the second last inequality holds since $E(N) \leq \sqrt{2 \log N}$, and the last inequality holds since by Assumption~\ref{assume:gaussian-tail}, $\mu_i$ is sampled from a distribution with bounded second moment $\E[\mu_i^2] \leq M_\mu$.

  And for the second term, we have:
  \begin{align*}
    \E\left[ \max_{i \in [K]} \hat{V}_N(X_i)^2 \right] 
    & \leq K  \E\left[ \hat{V}_N(X_i)^2 \right] \\
    & \leq 2K \E\left[ \hat{\mu}_i^2 \right] + 2K E(N)^2 \E\left[ \hat{\sigma}_i^2 \right]
  \end{align*}

  Since $\hat{\sigma}^2 = C \hat{\sigma}^2(D_{tail})$ and $\hat{\mu} = \hat{\mu}(D_{tail}) - C \hat{\sigma}$ for some constant $C$ depending on $(\alpha, \sigma_{min}, \sigma_{max})$, we only need to bound the moments of the empirical mean and variance on the tail dataset $D_{tail}$
  . By
  Assumption~\ref{assume:gaussian-tail}, we know $\E[R^2] \leq C_R$, and by properties of variance and mean, we have:
  $$ \E[\hat{\sigma}^2(D_{tail})] \leq C_R, \quad \E[\hat{\mu}(D_{tail})^2] \leq C_R $$

  Therefore, there exists a constant $c'$ depending on $(\sigma_{min}, \sigma_{max}, \alpha, C_R)$ such that:

  $$\E\left[ \max_{i \in [K]} \hat{V}_N(X_i)^2 \right]   \leq c' K \log N$$
  
  Considering that $K + \log N \leq 2 K \log N$ for sufficiently large $N$, we know there exists a constant $c$ depending on $(\sigma_{min}, \sigma_{max}, \alpha, C_R, M_\mu)$ such that:
  \begin{align*}
    \E\left[ \max_{i \in [K]} \epsilon_i^2 \right] \leq c K \log N
  \end{align*}

  Therefore:
  
  \begin{align*}
    \E\left[ \max_{i \in [K]} |\epsilon_i| \cdot \Ind(E^C)\right]& \leq \sqrt{cK \log N} \cdot \sqrt{K} e^{-m/2c_1} \\ 
    &= C K \sqrt{\log N} \exp(-C' m)
  \end{align*}
  
  where $C, C'$ depends on $(\sigma_{min}, \sigma_{max}, \alpha, C_R, M_\mu)$.

  Combining Term 1 and Term 2 gives:
  \begin{align*}
    \E\left[ \max_{i \in [K]} |\epsilon_i| \right] \leq c\left(\sqrt{\frac{\log N \log K}{m}} + K \sqrt{\log N} \exp(-C' m)  \right)
  \end{align*} 

  when $m \geq c_0 \log K$ for constant $c_0$ depending on $(\sigma_{min}, \sigma_{max}, \alpha, C_R, M_\mu)$, we know the second term is dominated by the first term,
  and thus we have:
  
  \begin{align*}
    \E\left[ \max_{i \in [K]} |\epsilon_i| \right] \leq c_1 \sqrt{\frac{\log N \log K}{m}}
  \end{align*}
  
  where $c_1$ is a constant depending on $(\sigma_{min}, \sigma_{max}, \alpha, C_R, M_\mu)$.
  This concludes the proof.
\end{proof}

Now we are ready to prove Theorem~\ref{thm:regret_bound}.

\begin{proof}[Proof of Theorem~\ref{thm:regret_bound}]

   We apply a coupling argument. Let $\xi = (s_1, s_2, \dots)$ be an infinite sequence of states sampled from the parameter prior distribution,
  and we condition on a fixed realization of $\xi$.

   Conditioned on $\xi$, we know that the states sampled by Scaling-Law Guided(SLG) algorithm with budget $N$ and the oracle algorithm $\A^*$ with budget $\lfloor N/2m \rfloor$
  are the same since $K = \lfloor N/2m \rfloor$ candidates are sampled in both algorithms. Denote the sampled states set as $S_K = \{s_1, s_2, \dots, s_K\}$, since
  $N \geq 4(m+1)$, we know $ N - Km \geq \lfloor N/2m \rfloor$.

   Then we can bound the performance of the oracle algorithm $\A^*$ as follows:
  \begin{align*}
    \E[V_{\lfloor N/2m \rfloor}(\A^*) \mid \xi] & = \max_{s \in S_K} V_{\lfloor N/2m \rfloor}(s) \\
    & \leq  \max_{s \in S_K} V_{N -Km}(s)
  \end{align*}
  We denote $s^* = \arg\max_{s \in S_K} V_{N-Km}(s)$ as the optimal state, and $\hat{s} = \arg\max_{s \in S_K} \hat{V}_{N-Km}(s)$ as the state selected by SLG algorithm based on the estimated expected maximal reward.

  Then we can bound the difference between the oracle algorithm $\A^*$ and SLG algorithm conditioned on $\xi$ as follows:
  \begin{align*}
    \E[V_{\lfloor N/2m \rfloor}(\A^*) - V_{N}(\A_{\text{SLG}}) \mid \xi] & \leq  V_{N - Km}(s^*) - V_{N - Km}(\hat{s})  \\
    & = V_{N - Km}(s^*) - \hat{V}_{N - Km}(s^*) + \hat{V}_{N - Km}(s^*) - \hat{V}_{N - Km}(\hat{s}) + \hat{V}_{N - Km}(\hat{s}) - V_{N - Km}(\hat{s}) \\
    & \leq 2 \max_{s \in S_K} |\epsilon_s|
  \end{align*}

  where $\epsilon_s = \hat{V}_{N - Km}(s) - V_{N - Km}(s)$ is the estimation error for each candidate state $s$. Since $K = \lfloor N/2m \rfloor \leq N-Km$ and $m \geq c_0 \log K$
  , we can take 
  expectation over $\xi$ on both sides and apply Lemma~\ref{lemma:selection-regret} to get:
  
  \begin{align*}
    V_{\lfloor N/2m \rfloor}(\A^*) - V_{N}(\A_{\text{SLG}}) \leq \E\left[ 2 \max_{s \in S_K} |\epsilon_s| \right] \leq c_1 \sqrt{\frac{\log (N-Km) \log K}{m}} \leq c_1 \frac{\log N}{\sqrt{m}}
  \end{align*}

  for some constant $c_1$ depending on $(\sigma_{min}, \sigma_{max}, \alpha, C_R, M_\mu)$. This completes the proof.

\end{proof}

\section{Formal Regret Analysis under Regularity Conditions}
\label{appendix:formal_regret_analysis}

In this section, we formalize the intuition presented in Section~\ref{subsubsec:regret} regarding the vanishing regret of SLG against the oracle $\A^*$ operating at the same budget $N$.

We first introduce the necessary regularity conditions on the latent parameter distributions.

\begin{assumption}
  \label{assum:regularity}
  For random state $s$ generated from prompt $x$, the latent parameters $(\mu(s), \sigma(s))$ satisfies
  \begin{enumerate}
    \item $\mu (s)$ and $\sigma (s)$ are independent.
    \item The random variable $\sigma (s)$ has sufficient mass near its maximum: $P(\sigma (s) \geq \sigma_{\max} - \delta) \geq \delta^\alpha$ for some $\delta > 0$.
    \item The expected maximum of the $\mu (s)$ scales as:
    $$\Big| \E[\max_{i \in [N]} \mu_i] - (\mu_0 + C_\mu (\log N)^\gamma)  \Big| \leq \mathcal{O}( (\log N)^{-\beta} )$$
    for constants $\mu_0, C_\mu, \beta > 0$ and tail index $\gamma \in (0,1)$.
  \end{enumerate}
\end{assumption}

And we have the following corollary of Theorem~\ref{thm:regret_bound}:

\begin{corollary}
  \label{cor:regret_to_N}
  Under Assumption~\ref{assume:gaussian-tail} and Assumption~\ref{assum:regularity},
  for sufficiently large $N$,
  by choosing the parameters $K \asymp N (\log N)^{-3}$ and $m \asymp (\log N)^{3}$, we have
  \begin{align*}
    V_{N}(\mathcal{A}^*) - V_N(\A_{\text{SLG}}) \leq \frac{c\log \log N}{(\log N)^{\min\{\beta, 1-\gamma, 1/2 \}}} .
  \end{align*}
  for some constant $c > 0$ depending on the model parameters $(\alpha, \sigma_{\min}, \sigma_{\max}, C_R, M_{\mu})$.
\end{corollary}

While Assumption~\ref{assum:regularity} posits independence between $\mu(s)$ and $\sigma(s)$ for analytical simplicity, this condition can be relaxed. 
As shown in the proof provided in next section, $\A_{\text{SLG}}$ predominantly selects states with variance near $\sigma_{\max}$ as $N$ grows. Consequently, it suffices to assume that the tail-scaling property of $\mu(s)$ (Condition 3) holds \textit{conditional} on the variance being close to $\sigma_{\max}$. Intuitively, this only requires that the mean and variance are not strongly negatively correlated in the high-variance region.
    
Conditions (2) and (3) are mild regularity conditions and are satisfied by many standard distributions. For instance, if the means $\mu(s)$ follow a Gaussian distribution and the variances $\sigma(s)$ are uniformly distributed, then the regularity conditions hold with $\alpha = 1$, $\gamma = 1/2$, and any $\beta > 0$.

We provide the proof of Corollary~\ref{cor:regret_to_N} in next section.

\subsection{Proof of Corollary~\ref{cor:regret_to_N}}
\label{appendix:proof-corollary-regret-to-N}

To help us evaluate the difference between $\A_{\text{SLG}}$ with exploration width $K$ and the full-information oracle $\A^*$, we first define an intermediate 
oracle algorithm $\hat{\A}$ that knows the true expected maximal reward $V_N(s)$ for each candidate state $s$ but only selects among the $K$ candidates sampled beforehand.

\begin{definition}
   The Oracle $\hat{\A}$ first samples a pool of $K$ candidate states $\mathcal{S}_K = \{s_1, \dots, s_K\}$ from the prompt $x$. It then identifies the optimal state $s^{*} = \operatorname*{argmax}_{s \in \mathcal{S}_K} V_N(s)$ using perfect parameter information and dedicates the entire budget $N$ to sampling from it.
\end{definition}

Then we can compare the performance of $\A_{\text{SLG}}$ with $\hat{\A}$ first, which is shown in the following lemma.

\begin{lemma}
\label{lemma:regret_to_hat_A}
Under Assumptions \ref{assume:gaussian-tail},
let $N$ be the total sampling budget and $m$ be the estimation sample size. 
When we choose the sample size scales as $m = c_0 \left( \frac{N}{K} \right)^{2/3} \log(N)$ for some constant $c_0 > 0$, 
the performance gap between $\A_{\text{SLG}}$ and the oracle $\hat{\A}$ satisfies:
\begin{equation}
    V_{N}(\hat{\mathcal{A}})- V_N(\A_{\text{SLG}}) \leq C \left( \frac{K}{N} \right)^{1/3} \sqrt{\log N}
\end{equation}
where $C$ is a constant depending on $(\sigma_{min}, \sigma_{max}, \alpha, C_R, M_\mu)$.
\end{lemma}

\begin{proof}[Proof of Lemma~\ref{lemma:regret_to_hat_A}]
  We apply the same coupling argument as in the proof of Theorem~\ref{thm:regret_bound}.
  Let $\xi = (s_1, s_2, \dots)$ be an infinite sequence of states sampled from the parameter prior distribution,
  and we condition on a fixed realization of $\tau$.


Conditioned on $\xi$, we know that the states sampled by Scaling-Law Guided(SLG) algorithm and the oracle algorithm $\hat{\A}$ are the same, and we denote 
$s^*$ as the optimal state with the highest true expected maximal reward among the $K$ candidates, i.e., $s^* = \arg\max_{i \in [K]} V_N(s_i)$,
and $\hat{s}$ as the state selected by SLG algorithm based on the estimated expected maximal reward, i.e., $\hat{s} = \arg\max_{i \in [K]} \hat{V}_N(s_i)$.

Then we can decompose the performance gap between SLG and the oracle algorithm $\hat{\A}$ as follows:
\begin{align*}
  V_{N}(\hat{\A}) - V_N(\A_{\text{SLG}}) 
  & \leq \E[ V_N(s^*) - V_{N - Km}(\hat{s}) ] \\ 
  & = \E[V_N(s^*) - V_N(\hat{s})] + \E[V_N(\hat{s}) - V_{N - Km}(\hat{s})]
\end{align*}

where the inequality holds because conditioned on $\tau$, the expected 
reward of SLG algorithm is lower bounded by the expected maximal reward of the selected state $\hat{s}$ with $N - Km$ samples.

For the first term, using the same technique as in the proof of Theorem~\ref{thm:regret_bound} and by applying Lemma~\ref{lemma:selection-regret}, we have:
\begin{align*}
  \E[V_N(s^*) - V_N(\hat{s})] \leq C_1 \sqrt{\frac{\log N \log K}{m}}
\end{align*}

for some constant $C_1$ depending on $(\sigma_{min}, \sigma_{max}, \alpha, C_R, M_\mu)$.

For the second term, we analyze the regret caused by the reduced sample size. Using the scaling law property, we have:
\begin{align*}
 V_N (\hat{s}) - V_{N - Km} (\hat{s}) 
  & =  \sigma(\hat{s}) (E(N) - E(N-Km)) \\
  & \leq \sigma_{max} (E(N) - E(N-Km))
\end{align*}
where $E(N)$ is the expected maximum of $N$ standard normal variables.
Since the above inequality holds for any fixed $\tau$, we can take expectation over $\tau$ on both sides, and since $E(n) - E(m) \leq 2(\sqrt{\log n} - \sqrt{\log m})$ for any $n > m \geq 3$, we have:
\begin{align*}
   \E\left[ V_N (\hat{s}) - V_{N - Km} (\hat{s}) \right] & \leq 2 \sigma_{max} (\sqrt{\log N} - \sqrt{\log (N - Km)}) \\
   & \leq \sigma_{max} \frac{Km}{(N-Km) \sqrt{\log (N - Km)}}
\end{align*}

And considering that $N - Km \geq \frac{N}{2}$, there exists a constant $C_2$ depending on $\sigma_{max}$ such that:
\begin{align*}
   \E\left[ V_N (\hat{s}) - V_{N - Km} (\hat{s}) \right] & \leq C_2 \frac{K}{N \sqrt{\log N}} m
\end{align*}

Combining the two parts, the total performance gap between $\hat{\A}$ and $\A_{\text{SLG}}$ is bounded by:
\begin{align*}
  V_{N}(\hat{\A}) - V_N(\A_{\text{SLG}})
   &\leq C_1 \sqrt{\frac{\log N \log K}{m}} + C_2 \frac{K m}{N \sqrt{\log N}}
\end{align*}

To minimize the upper bound, we choose the sample size $m$ to balance the order of the two terms,  which suggests setting $m = c_0 (N/K)^{2/3} \log N$ for some constant $c_0 > 0$.
Substituting this choice of $m$ back into the inequality yields:
\begin{align}
  \label{eq:hatA-SLG-regret}
  V_{N}(\hat{\A}) - V_N(\A_{\text{SLG}}) & \leq C \left( \frac{K}{N} \right)^{1/3} (\log N)^{1/2}
\end{align}

where $C$ is a universal constant depending on the problem parameters. This completes the proof.

\end{proof}

And then it suffices to compare the performance of $\hat{\A}$ with the full-information oracle $\A^*$, which is shown in the following lemma.

\begin{lemma}
\label{lemma:oracle-comparison}
  Under Assumption~\ref{assume:gaussian-tail} and Assumption~\ref{assum:regularity}
  , we can bound the performance gap between the full-information oracle $\A^*$ and the restricted oracle $\hat{\A}$ with $K \asymp N(\log N)^{-3}$ as:
  \begin{align*}
    V_N(\A^*) - V_N(\hat{\A}) \leq C \frac{\log \log N}{(\log N)^{\min\{\beta, 1-\gamma \}}}.
  \end{align*}
  where $C$ is a constant depending on $(\sigma_{min}, \sigma_{max}, \alpha, C_R, M_\mu, C_\mu, \mu_0, \beta, \gamma)$.
\end{lemma}

\begin{proof}[Proof of Lemma~\ref{lemma:oracle-comparison}]
  Let $\mathcal{S}_N$ denote the set of states sampled by the oracle strategy $\mathcal{A}^*$. We upper bound the performance of $\mathcal{A}^*$ as follows:
  \begin{align*}
  V_{N}(\A^*) & = \E\left[ \max_{s \in \mathcal{S}_N} V_N(s) \right] \\
    & = \E\left[ \max_{s \in \mathcal{S}_N} \left(\mu(s) + \sigma(s) E(N)\right) \right] \\
    & \leq \E \left[ \max_{s \in \mathcal{S}_N} \mu(s) \right] + E(N) \sigma_{\max} \\
    & \leq \mu_0 + C_{\mu} (\log N)^{\gamma} + E(N)\sigma_{\max} + C(\log N)^{-\beta},
  \end{align*}
  where $E(N)$ is the expected maximum of $N$ standard normal variables, and the last inequality follows from Assumption~\ref{assum:regularity}.
  
  To establish a lower bound for $\hat{\mathcal{A}}$, we consider a weaker oracle algorithm $\hat{\mathcal{A}}_{\delta}$ for some $\delta \in (0,1)$, constructed as follows:
  \begin{itemize}
    \item $\hat{\mathcal{A}}_{\delta}$ draws $K$ states $\{s_1, \dots, s_K\}$ from the prior. It has access to the true values $\{V_N(s_i)\}_{i=1}^K$.
    \item Define the candidate set $\mathcal{S}_\delta = \{ s_i \mid \sigma(s_i) \geq \sigma_{\max} - \delta, \, i \in [K] \}$. $\hat{\mathcal{A}}_{\delta}$ selects the state $s^*$ that maximizes the value within $\mathcal{S}_\delta$, i.e., $s^* = \arg\max_{s \in \mathcal{S}_\delta} V_N(s)$. If $\mathcal{S}_\delta$ is empty, an arbitrary state is selected.
    \item The algorithm samples $N$ times from $s^*$ and outputs the best response.
  \end{itemize}

  Since the search space of $\hat{\mathcal{A}}_{\delta}$ is a subset of the space explored by $\hat{\mathcal{A}}$ (with additional constraints), 
  we have $V_N(\hat{\mathcal{A}}) \geq V_N(\hat{\mathcal{A}}_\delta)$. Thus, it suffices to lower bound the performance of $\hat{\mathcal{A}}_\delta$.

  By Assumption~\ref{assum:regularity}, the probability that a sampled state has variance at least $\sigma_{\max} - \delta$ is lower bounded by $\delta^{\alpha}$. This implies $\E[|\mathcal{S}_\delta|] \geq \delta^{\alpha} K$. Define the event $E = \{ |\mathcal{S}_\delta| \geq \frac{1}{2}\delta^\alpha K \}$. By a standard Chernoff bound, we have:
  \[ \mathbb{P}(E) \geq 1 - \exp(-C' \delta^{\alpha} K), \]
  for some constant $C' > 0$. 
  
  Since $\mu$ and $\sigma$ are independent under the prior, conditioning on $E$ (which depends only on $\sigma$) does not alter the distribution of $\mu$ for states in $\mathcal{S}_\delta$. We invoke the mean scaling assumption with effective sample size $|\mathcal{S}_\delta|$:
  \begin{align*}
    \E\left[\max_{s \in \mathcal{S}_\delta} \mu(s) \mid E\right] & \geq \mu_0 + C_{\mu} \left(\log ( \delta^{\alpha} K / 2)\right)^{\gamma} - C \left(\log ( \delta^{\alpha} K / 2)\right)^{-\beta}.
  \end{align*}
  
  Conditioned on $E$, every state in $\mathcal{S}_\delta$ satisfies $\sigma(s) \geq \sigma_{\max} - \delta$. We lower bound the expected value:
  \begin{align*}
    V_N(\hat{\mathcal{A}}_\delta) & = \E[ R_N(\hat{\mathcal{A}}_\delta) ]  \geq \E[R_N(\hat{\mathcal{A}}_\delta) \mid E] \mathbb{P}(E) \\
    & \geq \left( \mu_0 + C_{\mu} (\log ( \delta^{\alpha} K / 2))^{\gamma} - C (\log ( \delta^{\alpha} K / 2))^{-\beta} + E(N) (\sigma_{\max} - \delta) \right) \left( 1 - \exp(-C' \delta^{\alpha} K) \right).
  \end{align*}

  We set $\delta = (\log N)^{-2}$. Recalling that $K = C_K N (\log N)^{-3}$, the failure probability $\exp (-C' \delta^{\alpha} K)$ decays polynomially in $N$ (specifically as $O(N^{-c})$ for some $c>0$), which is negligible compared to polylogarithmic terms.
  
  Let $P = \frac{N}{\delta^\alpha K/2}$. We have $\log P = O(\log \log N)$ and $\log (\delta^\alpha K / 2) = \log N - \log P$. For sufficiently large $N$, using the inequality $(1-x)^\gamma \geq 1 - x$ for $x \in [0,1)$ and $\gamma < 1$, the lower bound simplifies to:
  \begin{align*}
    V_N(\hat{\mathcal{A}}_\delta)
    & \geq \mu_0 + C_{\mu} (\log N)^\gamma \left(1 - \frac{\log P}{\log N}\right)^{\gamma} - C'' (\log N)^{-\beta} + E(N) \sigma_{\max} - z_N \delta - O(N^{-c}) \\
    & \geq \mu_0 + C_{\mu} (\log N)^\gamma \left(1 - \frac{\log P}{\log N}\right) - C'' (\log N)^{-\beta} + E(N) \sigma_{\max} - \frac{z_N}{(\log N)^2} - O(N^{-c}).
  \end{align*}

  Subtracting this from the upper bound of $V_N(\mathcal{A}^*)$, we obtain:
  \begin{align*}
    V_{N}(\A^*) - V_{N}(\hat{\A}) & \leq V_{N}(\A^*) - V_{N}(\hat{\A}_\delta) \\
    & \leq C_\mu (\log N)^{\gamma-1} \log P + \frac{E(N)}{(\log N)^2} + C (\log N)^{-\beta} + O(N^{-c}).
  \end{align*}

  Using the facts that $\log P = O(\log \log N)$ and $E(N) \leq \sqrt{2 \log N}$, the gap is dominated by:
  \begin{align*}
     V_{N}(\A^*) - V_{N}(\hat{\A})
      & \leq O\left( \frac{\log \log N}{(\log N)^{1-\gamma}} \right) + O\left( \frac{1}{(\log N)^{3/2}} \right) + O\left( \frac{1}{(\log N)^\beta} \right)\\
     & \leq C \frac{\log \log N}{(\log N)^{\min\{\beta, 1-\gamma\}}},
  \end{align*}
  for some constant $C > 0$.
\end{proof}

Then we are ready to prove Corollary~\ref{cor:regret_to_N}.

\begin{proof}[Proof of Corollary~\ref{cor:regret_to_N}]

  We first decompose the performance gap between $\A^*$ and Scaling-Law Guided (SLG) algorithm as follows:
  \begin{align*}
    V_{N}(\A^*) - V_N(\A_{\text{SLG}}) & =  V_{N}(\A^*) - V_N(\hat{\A}) + V_N(\hat{\A}) - V_N(\A_{\text{SLG}})
  \end{align*}

  Since the choice of $K,m$ satisfies the conditions in Lemma~\ref{lemma:regret_to_hat_A}, we know by Equation~\eqref{eq:hatA-SLG-regret} that:
  \begin{align*}
    V_N(\hat{\A}) - V_N(\A_{\text{SLG}})
    & \leq C \left( \frac{K}{N} \right)^{1/3} (\log N)^{1/2}\\
    & \leq C_1 \frac{1}{(\log N)^{1/2}}
  \end{align*}

  for some constant $C_1 > 0$. And by Lemma~\ref{lemma:oracle-comparison}, we have:

  \begin{align*}
    V_{N}(\A^*) - V_N(\hat{\A})
    & \leq C_2 \frac{\log \log N}{(\log N)^{\min\{\beta, 1-\gamma \}}}
  \end{align*}

  for some constant $C_2 > 0$. Combining the two parts, we conclude that:
  \begin{align*}
    V_{N}(\A^*) - V_N(\A_{\text{SLG}})
    & \leq C_2 \frac{\log \log N}{(\log N)^{\min\{\beta, 1-\gamma \}}} + C_1 \frac{1}{(\log N)^{1/2}} \\
    & \leq C \frac{\log \log N}{(\log N)^{\min\{\beta, 1-\gamma, 1/2 \}}}
  \end{align*}
  for some constant $C > 0$ depending on the problem parameters. This completes the proof.

\end{proof}

\section{Proofs in Section~\ref{subsubsec:gaussian_examples}}
\label{appendix:gaussian_examples_proofs}

\subsection{Proof of Proposition~\ref{prop:homogeneous_gap}}
\label{appendix:proof-homogeneous-gap}

We analyze the performance gap between SLG and Best-of-N algorithms under the homogeneous noise setting as stated in Proposition~\ref{prop:homogeneous_gap}.

\begin{proof}[Proof of Proposition~\ref{prop:homogeneous_gap}]
  We first analyze the performance of our algorithm SLG and the Best-of-N algorithm under the homogeneous noise setting. 

  If we denote $S_{\text{SLG}}$ as the set of states selected by SLG algorithm
  (where we know $|S_{\text{SLG}}| = K$), 
  and $\hat{s} = \arg\max_{s \in S_{\text{SLG}}} \hat{V}_N(s)$ as the candidate with the highest estimated value among the selected states, then since
  the variance $\sigma(s)=\sigma_1$ for all $s$, we have:
  \begin{align*}
    \hat{s}= \arg\max_{s \in S_{\text{SLG}}} \hat{V}_N(s) & = \arg\max_{s \in S_{\text{SLG}}} \hat{\mu}(s)
  \end{align*}
  and we can lower bound the expected performance of SLG algorithm as follows:
  \begin{align*}
    V_N(\A_{\text{SLG}}) & \geq \E[ V_{N - Km}(\hat{s}) ] \\
    & = \E[ \mu(\hat{s}) + \sigma_1 E(N - Km) ] \\
    & = \E[ \mu(\hat{s}) ] + \sigma_1 E(N - Km)
  \end{align*}
  Because we know $\hat{\mu}(s) \mid \mu(s) \sim \mathcal{N}(\mu(s), \sigma_1^2/m)$ and $\mu(s) \sim \mathcal{N}(\mu_0, \sigma_0^2)$, by the properties of Gaussian distribution, we have:
  \begin{align*}
    \E[\mu(\hat{s})] & = \E_{\hat{\mu}} \left[ \E[ \mu(\hat{s}) \mid \hat{\mu}(\hat{s}) ]  \right] \\
    & = \E_{\hat{\mu}} \left[ \mu_0 + \frac{\sigma_0^2}{\sigma_0^2 + \sigma_1^2/m} (\hat{\mu}(\hat{s}) - \mu_0) \right] \\
    & = \mu_0 + \frac{\sigma_0^2}{\sigma_0^2 + \sigma_1^2/m} \E_{\hat{\mu}}[ \hat{\mu}(\hat{s}) - \mu_0 ] \\
    & = \mu_0 + \frac{\sigma_0^2}{\sigma_0^2 + \sigma_1^2/m} \sqrt{\sigma_0^2 + \sigma_1^2/m} E(K) 
  \end{align*}

  where the last equality follows from the fact that $\hat{\mu}(s) \sim \mathcal{N}(\mu_0, \sigma_0^2 + \sigma_1^2/m)$ for all $s$ and $\hat{s}$ is the maximizer among $K$ i.i.d. samples.
  If we denote $t = \frac{\sigma_1}{\sigma_0}$, then we can simplify the above expression as:
  \begin{align*}
    \E[\mu(\hat{s})] & = \mu_0 + \frac{\sigma_0}{\sqrt{1 + t^2/m}} E(K)
  \end{align*}
  which leads to the following lower bound for SLG algorithm:
  \begin{align*}
    V_N(\A_{\text{SLG}}) & \geq \mu_0 + \sigma_0 \left[ \frac{1}{\sqrt{1 + t^2/m}} E(K) + t E(N - Km) \right]
  \end{align*}
  For the Best-of-N algorithm, since the reward distribution of the prompt $x$ is $\mathcal{N}(\mu_0, \sigma_0^2 + \sigma_1^2)$, we have:
  \begin{align*}
    V_N(\A_{\text{BoN}}) & = \E[\max_{R_1, \dots, R_N \simiid \mathcal{N}(\mu_0, \sigma_0^2 + \sigma_1^2)} R_i ] = \mu_0 + \sigma_0 \sqrt{1 + t^2} E(N)
  \end{align*}

  So the performance gap between the two algorithms can be lower bounded as:
\begin{align}
    V_N(\mathcal{A}_{\text{SLG}}) - V_N(\mathcal{A}_{\text{BoN}}) & \geq \sigma_0 \left[ \frac{1}{\sqrt{1 + t^2/m}} E(K) + t E(N - Km) - \sqrt{1 + t^2} E(N) \right].
\end{align}
To proceed, we define a surrogate function to approximate the gap, and then derive a lower bound for this surrogate. Let:
\begin{align*}
    G(t,N) = \frac{1}{\sqrt{1 + t^2/m}} \sqrt{2\log K} + t \sqrt{2\log (N - Km)} - \sqrt{1 + t^2} \sqrt{2\log N}.
\end{align*}
By choosing $m = t^2 \log N$ and $K = \frac{N}{(1+t)m} = \frac{N}{(1+t) t^2 \log N}$, we have:
\begin{align*}
    \frac{1}{\sqrt{1 + t^2/m}} \sqrt{2\log K} & = \frac{1}{\sqrt{1 + 1/\log N}} \sqrt{2 \log N - 2 \log \log N - 2 \log (t^2(1+t))} \\
    & \geq \frac{1}{\sqrt{1 + 1/\log N}} \sqrt{2\log N}\left( 1 - \frac{\log \log N + \log (t^2(1+t))}{\log N} \right) \\
    & \geq \frac{1}{\sqrt{1 + 1/\log N}} \sqrt{2\log N} - C_1 \frac{\log \log N}{\sqrt{\log N}}
\end{align*}
for some constant $C_1 > 0$ depending on $t$. Similarly, we have:
\begin{align*}
    t \sqrt{2 \log(N-Km)} & = t \sqrt{2 \log N  - 2 \log (1+ 1/t)} \\
    & \geq t \sqrt{2 \log N} \left( 1 - \frac{\log(1 + 1/t)}{\log N} \right) \\
    & \geq t \sqrt{2 \log N} - C_2 \frac{1}{\sqrt{\log N}}
\end{align*}
for some constant $C_2 > 0$ depending on $t$. Combining these parts yields:
\begin{align*}
    G(t,N) & \geq \left( \frac{1}{\sqrt{1 + 1/\log N}} + t - \sqrt{1 + t^2} \right) \sqrt{2\log N} - C \frac{\log \log N}{\sqrt{\log N}}
\end{align*}
for some constant $C > 0$ depending on $t$. When $N \geq \exp(1 + 1/t)$, we observe:
\begin{align*}
    \frac{1}{\sqrt{1 + 1/\log N}} + t - \sqrt{1 + t^2} & \geq \frac{1}{\sqrt{1 + \frac{t}{t+1}}} + t - \sqrt{1 + t^2} \\
    & \geq \left(1 - \frac{1}{2}\frac{t}{t+1}\right) + t - \sqrt{1+t^2} \\
    & \geq \frac{t}{t+1} - \frac{1}{2}\frac{t}{t+1} = \frac{1}{2} \frac{t}{t+1},
\end{align*}
where the last inequality follows from the fact that $t+1 - \sqrt{1+t^2} \geq \frac{t}{t+1}$ for all $t > 0$. Therefore:
\begin{align*}
    G(t,N) & \geq \frac{1}{2} \frac{t}{t+1} \sqrt{2\log N} - C \frac{\log \log N}{\sqrt{\log N}}
\end{align*}
for $N \geq \exp(1 + 1/t)$.

It suffices to analyze the approximation error between the original gap and the surrogate function, which arises from approximating $E(N)$ by $\sqrt{2\log N}$. Using the known bound:
\begin{align*}
    \left| E(N) - \sqrt{2\log N} \right| & \leq C_3 \frac{\log \log N}{\sqrt{\log N}}
\end{align*}
for some constant $C_3 > 0$ and all $N \geq 3$, we have:
  \begin{align*}
    & \quad \left| V_N(\A_{\text{SLG}}) - V_N(\A_{\text{BoN}}) - \sigma_0 G(t,N) \right| \\
    & \leq \sigma_0 \left( \frac{1}{\sqrt{1 + t^2/m}} \left| E(K) - \sqrt{2\log K} \right| + t \left| E(N - Km) - \sqrt{2\log (N - Km)} \right| + \sqrt{1 + t^2} \left| E(N) - \sqrt{2\log N} \right| \right) \\
    & \leq \sigma_0 \left( C_3 \frac{\log\log K}{\sqrt{\log K}} + t C_3 \frac{\log\log (N - Km)}{\sqrt{\log (N - Km)}} + \sqrt{1 + t^2} C_3 \frac{\log\log N}{\sqrt{\log N}} \right) \\
    & \leq C_4 \frac{\log\log N}{\sqrt{\log N}}
  \end{align*}
for some constant $C_4 > 0$ depending on $t$.

Combining the above results, we observe that the lower bound consists of a dominant term of order $\sqrt{\log N}$ and a negative error term of order $\frac{\log \log N}{\sqrt{\log N}}$. 
Since $\lim_{N \to \infty} \frac{\log \log N}{\log N} = 0$, there exists a threshold $N_0$ (dependent on $t$) such that for all $N \geq N_0$, the error term is at most half of the dominant term. 
Therefore, for $N \geq N_0$:
\begin{align*}
    V_N(\mathcal{A}_{\text{SLG}}) - V_N(\mathcal{A}_{\text{BoN}}) & \geq \frac{t}{4(t+1)} \sigma_0 \sqrt{2\log N}
\end{align*}
This completes the proof.
\end{proof}

\subsection{Proof of Corollary~\ref{cor:compute_amplification}}
\label{appendix:proof-compute-amplification}

We now prove Corollary~\ref{cor:compute_amplification} by translating the performance gap into an equivalent budget amplification.

\begin{proof}[Proof of Corollary~\ref{cor:compute_amplification}]
We aim to find a constant $\gamma > 0$ such that for $N \geq N_0$ (as defined in Proposition~\ref{prop:homogeneous_gap}), the inequality $V_N(\mathcal{A}_{\text{SLG}}) \geq V_{N^{1+\gamma}}(\mathcal{A}_{\text{BoN}})$ holds.

From Proposition~\ref{prop:homogeneous_gap}, the performance of the SLG algorithm is lower-bounded by the performance of the standard Best-of-$N$ baseline plus a strictly positive margin:
\begin{equation}
    V_N(\mathcal{A}_{\text{SLG}}) \geq V_N(\mathcal{A}_{\text{BoN}}) + \frac{t}{4(t+1)} \sigma_0 \sqrt{2\log N}.
    \label{eq:slg_gap}
\end{equation}
Now, consider the Best-of-$N$ baseline operating with an amplified budget $N^{1+\gamma}$. Its expected value is given by:
\begin{equation*}
    V_{N^{1+\gamma}}(\mathcal{A}_{\text{BoN}}) = \mu_0 + \sigma_0 \sqrt{1+t^2} E(N^{1+\gamma}).
\end{equation*}
The difference in value between the amplified baseline and the standard baseline is:
\begin{align*}
    \Delta V_{\text{BoN}} &= V_{N^{1+\gamma}}(\mathcal{A}_{\text{BoN}}) - V_N(\mathcal{A}_{\text{BoN}}) \\
    &= \sigma_0 \sqrt{1+t^2} \left( E(N^{1+\gamma}) - E(N) \right).
\end{align*}
Since we know $\E(n) - E(m) \leq 2(\sqrt{\log n} - \sqrt{\log m})$ for any $n > m \geq 3$, we can further bound the difference:
\begin{align*}
    E(N^{1+\gamma}) - E(N) &\leq 2(\sqrt{\log (N^{1+\gamma})} - \sqrt{ \log N}) \\
    &= 2(\sqrt{1+\gamma} - 1) \sqrt{\log N}.
\end{align*}
Substituting this back into the expression for $\Delta V_{\text{BoN}}$, we have:
\begin{equation*}
    \Delta V_{\text{BoN}} \leq 2\sigma_0 \sqrt{1+t^2} (\sqrt{1+\gamma} - 1) \sqrt{\log N} \leq \sigma_0 \sqrt{1+t^2} \gamma \sqrt{\log N}.
\end{equation*}
And when we choose $\gamma = \frac{\sqrt{2}t}{4(t+1)\sqrt{1+t^2}}$, we have:

\begin{align*}
  V_{N}(\mathcal{A}_{\text{SLG}}) &\geq V_{N^{1+\gamma}}(\mathcal{A}_{\text{BoN}}) - \Delta V_{\text{BoN}} + \frac{t}{4(t+1)} \sigma_0 \sqrt{2\log N}\\
  & \geq V_{N^{1+\gamma}}(\mathcal{A}_{\text{BoN}}) - \sigma_0 \sqrt{1+t^2} \gamma \sqrt{\log N} + \frac{t}{4(t+1)} \sigma_0 \sqrt{2\log N} \\
  & = V_{N^{1+\gamma}}(\mathcal{A}_{\text{BoN}})
\end{align*}

This completes the proof.

\end{proof}

\end{document}